\documentclass{amsart}
\usepackage[dvipsnames]{xcolor}
\usepackage{relsize}
\usepackage{amssymb,hyperref}
\usepackage{mathtools}
\usepackage{cleveref}
\usepackage{tikz} 
\usetikzlibrary{cd,arrows,shapes,shapes.multipart, intersections,decorations,decorations.markings,decorations.pathreplacing,fit}
\tikzset{thick/.style={line width=.6mm}}
	\tikzstyle{hugetensor}=[rounded rectangle,thick,draw=black,minimum width=40mm,minimum height = 3mm]
	\tikzstyle{squaretensor}=[rounded rectangle,thick,draw=black,minimum width=15mm,minimum height = 7mm]
	\tikzstyle{littletensor}=[circle,thick,draw=black,fill=red!30,minimum size=.5mm]
	\tikzstyle{tinytensor}=[circle,thick,draw=black,fill=red!30,inner sep=0pt,minimum size=6pt]

\renewcommand{\leq}{\leqslant}
\renewcommand{\geq}{\geqslant}
\renewcommand{\le}{\leqslant}

\theoremstyle{plain}
\newtheorem*{theorem*}{Theorem}
\newtheorem{proposition}{Proposition}
\newtheorem{theorem}{Theorem}
\newtheorem{lemma}{Lemma}
\newtheorem{corollary}{Corollary}

\newtheorem{definition}{Definition}
\newtheorem*{notation*}{Notation}

\theoremstyle{definition}
\newtheorem{remark}{Remark}
\newtheorem{example}{Example}

\newcommand{\Supp}{\operatorname{Supp}}
\newcommand{\<}{\langle}
\renewcommand{\>}{\rangle}
\newcommand{\set}[2]{\{#1\mid #2\}}

\newcommand{\op}{\textrm{op}}

\newcommand{\Prob}{\operatorname{Pr}}
\newcommand{\logPr}{\log\Prob}

\newcommand{\R}{\mathbb{R}}
\newcommand{\mL}{\mathcal{L}}
\newcommand{\mR}{\mathcal{R}}
\newcommand{\Fix}{\operatorname{Fix}}

\newcommand{\Rmin}{\mathbb{R}_{\min}}
\renewcommand{\Im}{\operatorname{Im}}
\newcommand{\Id}{\operatorname{Id}}
\newcommand{\Span}{\operatorname{Span}}
\usepackage{pgf}
\usepackage{tikz}
\usepackage{tikz-cd}
\usetikzlibrary{cd}
\usetikzlibrary{calc}
\usetikzlibrary{arrows}
\usetikzlibrary{shapes}
\usepackage{barycentric}
\title[Directed Metric Structures arising in Large Language Models]{Directed Metric Structures arising in Large Language Models}

\usepackage[colorinlistoftodos,bordercolor=orange,backgroundcolor=orange!20,linecolor=orange,textsize=scriptsize]{todonotes}

\author{St\'ephane Gaubert}
\author{Yiannis Vlassopoulos}

\address{SG: INRIA and CMAP, \'Ecole polytechnique, IP Paris, CNRS}
\email{Stephane.Gaubert@inria.fr}
\address{YV: ATHENA Research Center, Institute for Language and Speech Processing, Athens, Greece and IHES, Bures-sur-Yvette, France}
\email{yvlassop@gmail.com, yvlassop@ihes.fr}

\date{\today}

\begin{document}
\begin{abstract}
Large Language Models are transformer neural networks which are trained to produce a probability distribution on the possible next words to given texts in a corpus, in such a way that the most likely word predicted is the actual word in the training text. 
In this paper we find what is the mathematical structure defined by such conditional probability distributions of text extensions. Changing the view point from probabilities to -log probabilities we observe that the subtext order is completely encoded in a metric structure defined on the space of texts $\mL$, by -log probabilities. We then construct a  metric polyhedron $P(\mL)$  and an isometric embedding (called Yoneda embedding) of $\mL$ into $P(\mL)$ such that texts map to generators of certain special extremal rays. %
We explain that $P(\mL)$ is a $(\min,+)$ (tropical) linear span of these extremal ray generators. The generators also satisfy  a system of $(\min+)$ linear equations. We then show that $P(\mL)$ is compatible with adding more text and from this we derive an approximation of a text vector as a Boltzmann weighted linear combination of the vectors for words in that text. 
We then prove a duality theorem showing that texts extensions and text restrictions give isometric polyhedra (even though they look a priory very different). Moreover we prove that $P(\mL)$  is the lattice closure of (a version of) the so called, Isbell completion of $\mL$ which turns out to be the $(\max,+)$ span of the text extremal ray generators. All constructions have interpretations in category theory but we don't use category theory explicitly. The categorical interpretations are briefly explained in an appendix. In the final appendix we describe how the syntax to semantics problem could fit in a general well known mathematical duality.

\end{abstract}
\maketitle
\tableofcontents

\section{Overview}

Large Language Models (LLM) are transformer neural networks that are trained to compute the probability of the possible next words to a text  in such a way that the most probable next word predicted by the network, is the actual next word in the training text \cite{vaswani2017attention, GPT1, GPT2, GPT3}.

They are often characterized as ``just statistical models''. In this paper, continuing the approach introduced in \cite{BTV2021}, this time without explicitly using categories \footnote{We note though that all constructions and theorems have a categorical interpretation. For those already familiar with categories Appendix A provides a brief categorical explanation of constructions and results in this paper.}, we would like to make a proposal for what is the underlying mathematical structure these probability distributions 
actually encode and show the evidence and possible consequences.  We find that a rich structure is revealed if we change the point of view from probabilities to negative log probabilities. 
While it is entirely equivalent, this point of view can be reinterpreted as an asymmetric metric on texts. 

Indeed consider a set $\mL:=\{a_1,\dots,a_n\}$ whose elements are texts in the language. 
Equip $\mL$ with a poset structure, where $a_i\leq a_j$ if and only if $a_i$ is a subtext of $a_j$. Denote by $\Prob(a_j|a_i)$ the probability of extending a text $a_i$ to a text $a_j$. 
The probability is $0$ exactly when $a_i$ is not a subtext of $a_j$.
Our main assumption is that conditional probabilities of extensions of texts multiply i.e.
\begin{equation}
 \text{if } a_i\leq a_j\leq a_k  \text{ then } 
 \Prob(a_j|a_i)\Prob(a_k|a_j)=\Prob(a_k|a_i).
 \label{e-main-assum}
\end{equation}
We call the triple $(\mL,\leq,\Prob)$ a \textit{probabilistic language model}.

Next, recall the notion of a \textit{directed metric} $\delta$ on a set $X$. It is defined to be a function $\delta:X \times X \to (-\infty,\infty]$ which satisfies the triangle inequality and $\delta (a,a)=0$. Nevertheless, unless stated otherwise 
it does not have to be symmetric, $\delta(a,b)=0$ does not necessarily imply $a=b$ and $\delta$ can take the value $\infty$ and can also take negative values.
\footnote{We will see however that restricting to positive values is natural when the directed metric comes from probability distributions. }

Now we notice that the probabilistic language model $(\mL,\leq,\Prob)$ defines a directed  metric $d$ on the poset $\mL$ by

\begin{equation}
d(a_i,a_j) = \left\{
\begin{array}{ll}
  -\logPr(a_j|a_i) & \text{if } a_i \leq a_j,\\
  
  \infty & \text{if $a_i$ and $a_j$ are not comparable} .
\end{array}
\right.
\end{equation}

\textit{$(\mL,d)$ is then a directed metric space.
} 

Indeed 
\begin{equation}
 \text{if } a_i\leq a_j \leq a_k \text{ then } 
d(a_i,a_k)=d(a_i,a_j)+d(a_j,a_k)
\end{equation}
 and 
otherwise the triangle inequality is satisfied with at least one side being $\infty$ (\Cref{prop-1}). We see also that the metric $d$ fully determines the poset structure on $\mL$ (\Cref{coro-1}).

Although $d$
satisfies a rather degenerate form of the triangle inequality, it is enough to define a non trivial directed metric polyhedron $P(\mL)$   in $(\mathbb{R}\cup \{\infty\})^n$, inside which $\mL$ is isometrically embedded as a remarkable set of extremal rays.

Indeed, in \Cref{section-3} we define
\begin{equation}
P(\mL):=\{x \in (\mathbb{R}\cup \{\infty\})^n \text{\textbackslash} (\infty,\dots,\infty)| x_i\leq x_j +d(a_i,a_j)\}
\end{equation}
and thus the finite part of $P(\mL)$ is a polyhedron in 
$\mathbb{R}^n$.
Define the  Funk directed metric $D$ on $(\mathbb{R}\cup \{\infty\})^n \text{\textbackslash} (\infty,\dots,\infty) $ by
\begin{equation}
D(x,x'):=\max_i\{x'_i-x_i|x_i\neq \infty\}.
\end{equation}
Then $(P(\mL),D)$ becomes a directed metric space. The Funk directed
metric originates from Hilbert geometry~\cite{Funk}.

To understand the relevance of $P(\mL)$ note first (\Cref{prop-yoneda-alt})
that 
\begin{equation}
 Y:(\mL,d) \hookrightarrow (P(\mL),D), \text{ given by } Y(a_k):=d(-,a_k),
 \end{equation}
is an isometric embedding (called the Yoneda embedding). Moreover,
each
$Y(a_k)$ is a generator of an extremal ray of $P(\mL)$ (\Cref{th-1}).
To be precise about the term extremal ray here, 
consider the image $Q(\mL)$ of $P(\mL)$, under the coordinate wise map $x_i \to z_i:=e^{-{x_i}}$, i.e.
\begin{equation}
Q(\mL):=\{z:=(z_1,\dots z_n )\in [0,\infty)^n | z_i:=e^{-x_i} \text { for } x=(x_1,\dots, x_n) \in P(\mL)\}
\end{equation}
We see that 
\begin{equation}
Q(\mL):=\{z=(z_1,\dots z_n)\in [0,\infty)^n \text{\textbackslash} (0,\dots,0) |z_i\geq \Prob(a_j|a_i)z_j \}
\end{equation}
Denote by $e^{-Y(a_k)}$, the image of $Y(a_k)$ in $Q(\mL)$. Then (\Cref{th-1}) $e^{-Y(a_k)}$ is a generator of a usual extremal ray in the polyhedral cone $Q(\mL)$.
When we speak of extremal rays of $P(\mL)$ we always mean the subsets of $P(\mL)$ that map to extremal rays in the polyhedral cone $Q(\mL)$
  by the map $(x_i)\mapsto (e^{-x_i})$.

The polyhedral cone $Q(\mL)$ is a generalization of the {\em order polytope} studied by Stanley \cite{Stanley86}. The order polytope corresponds to the case where 
$\Prob(a_j|a_i)$ takes only the values $0$ or $1$ and $\mL$ is simply a poset.
Moreover, Stanley does not consider a cone, but rather the intersection of this cone with the unit box.
Lam and Postnikov~\cite{Postnikov06} defined an {\em alcoved polytope} to be a bounded cell of a Coxeter arrangement (of type $A_n$). The definition 
of $P(\mL)$ is similar, but we do not require the cell to be bounded.
Alcoved polytopes have been studied in tropical geometry, in relation
with metric spaces, see e.g.~\cite{Joswig2010,Tran2017}.

Moving on,  we prove in \Cref{prop-5} that 
\begin{equation}
\text{if }x=(x_1,\dots, x_n)\in P(\mL) 
\text{ then }
x_k=D(Y(a_k),x).
\end{equation}

We now see that the defining equations for $P(\mL)$ are exactly the triangle inequalities for $D$. Indeed 
\begin{equation}
x_i\leq x_j +d(a_i,a_j)\iff D(Y(a_i),x)\leq D(Y(a_j),x) +D(Y(a_i),Y(a_j)).
\end{equation}

Note also (\Cref{prop-3}) that we can think of the points $x\in P(\mL)$ as functions on $\mL$ (just as we can think of usual vectors as functions on a set). Indeed, if we
denote by $d_\mathbb{R}$ the Funk metric on $\mathbb{R}\cup \{\infty\}$, namely $d_\mathbb{R}(s,t)=t-s$ and $d_\mathbb{R}(\infty,\infty)=\max{\emptyset}=-\infty$ \footnote{We explain later in \Cref{re extended metric} that it is possible and useful in some cases to extend the values of a directed metric to $[-\infty,\infty]$ and this is one of the cases we do so.}
then 
\[P(\mL)=\{x:(\mL,d^t) \to ((-\infty,\infty],d_F)  | x \text{ is non-expansive.}\} \]
We then have that $x(a_i)=x_i=D(Y(a_i),x)$ (\Cref{prop-5}).

Therefore $P(\mL)$ can also be thought of as the set of maps
$x:\mL \to (-\infty,\infty]$ 
 which satisfy the triangle inequalities for the metric $D$ with respect to all the maps $Y(a_k):=d(-,a_k):\mL \to (-\infty,\infty]$ for $k=1,\dots, n$. It is therefore a kind of convex metric span of the $Y(a_k):=d(-,a_k)$. 
 
From the point of view of \textit{language semantics now, we consider 
$Y(a_k)=d(-,a_k)$ to represent the meaning of a text $a_k$ in terms of all the texts it contains,} (\Cref{Section Semantic spaces}) in accordance with the statistical semantics principal and as was advocated in \cite{BTV2021}.

Dually, we can consider the meaning of a text $a_k$ to be given by all the texts extending $a_k$, namely $d(a_k,-)$.
This is then encoded in 
\begin{equation}
\widehat{P}(\mL)=\{y\in (\R \cup\{\infty\})^n|y_i\leq y_j +d(a_j,a_i)\}.
\end{equation}
Indeed we have the isometric co-Yoneda embedding
\begin{equation}
    \widehat{Y}:\mL \hookrightarrow P(\mL) \text{ given by } 
\widehat{Y}(a_k):=d(a_k,-)
\end{equation}
and moreover 
$y_i=D(\widehat{Y}(a_i),y)$.
\begin{remark}
    \label{re semantic}
We said that $Y(a_k):=d(-,a_k)$ or $\widehat{Y}(a_k):=d(a_k,-)$ represent the meaning of a text $a_k$ but it's also the ``location'' of these vectors in the whole space $P(\mL)$ and $\widehat{P}(\mL)$ respectively. 
In particular if $w:=a_k$ is a word then $Y(w):=d(-,w)$ is supported only on $w$ so it does not seem to contain much information. However 
the relevant semantic information is in $D(Y(w),-)$.
Moreover we will see shortly that the fact that the vector $d(-,w)$ is in $P(\mL)$ means that it satisfies a whole system of equations (Eq 19, 20, \Cref{prop linear system}) with respect to other texts.
We will explain this system of equations, later in this overview when we describe  section~\ref{section-4}. 
\end{remark}

We already saw that $Y(a_k)$ is an extremal ray in $P(\mL)$ but it turns out that $P(\mL)$  has generally exponentially many additional extremal rays -- that are not in the image of $Y$.
We explicitly characterize the extremal rays of $P(\mL)$ as corresponding to connected lower sets of $(\mL,\leq)$ in \Cref{prop-diagscaling} and \Cref{prop-char-extreme}.
In fact, if $x$ is such an extremal ray and $l(x)$ denotes the corresponding lower set then we show that after a diagonal change of variables the new coordinates  of the extremal ray give the characteristic function of the lower set $l(x)$.

What distinguishes the lower sets corresponding to elements $Y(a_i)$ is that they are principal.

\textit{Therefore $P(\mL)$ can be considered as a space parameterizing semantics in the language. 
However, only the extremal rays corresponding to principal lower sets of $\mL$, correspond to texts.}
Notice now that the Funk metric $D$ on $P(\mL)$ defines a metric $D_Q$ on the polyhedral cone $Q(\mL)$ where, if $z,z' \in Q(\mL)$, then
\begin{equation}
D_Q(z,z'):=\max_i\{\log(\frac{z_i}{z'_i})|z'_i\neq 0\}.
\end{equation}
By definition  we have 
 $D_Q(z,z')=D(-\log z,-\log z') \text{ and }
 D(x,x')=D_Q(e^{-x},e^{-x'})$.
 Then the fact that $Y$ is an isometric embedding into $P(\mL)$ implies that 
\begin{equation} 
    e^{-Y}:(\mL,d) \to (Q(\mL),D_Q)  
\end{equation}
    is an isometric embedding.

In section 4 we explain the construction of $P(\mL)$ in terms of
tropical or $(\min,+)$ algebra. Recall that the $(\min,+)$ semifield $\R_{\min}$, is defined as $((-\infty, \infty],\min,+)$.
We show in \Cref{section-4} that $P(\mL)$ is generated by the vectors $Y(a_k)=d(-,a_k)$ as a $(\min,+)$ module. 
To see that, note first that $d$ is a directed metric if and only if it is a $(\min,+)$ projector. This is because if we let $d_{i,j}:=d(a_i,a_j)$ and define 
\begin{equation}
d_{\min}(x)_i:=\min_j\{d_{i,j}+x_j\},
\end{equation}then the triangle inequality and the fact that $d_{i,i}=0$, are equivalent to  
$d_{i,k}=\min\{d_{i,j}+d_{j,k}\}$
which is equivalent to 
$d_{\min}^2=d_{\min}$.

We then note that 
$P(\mL)=\{x|d_{\min}x=x\}.$
Indeed:
$x=d_{\min}x \iff x_i=\min_j\{d_{i,j}+x_j\}\iff x\in P(\mL).$

Introduce the notation $\Fix(d_{\min}):=\{x|d_{\min}x=x\}$. Since $d_{\min}$ is a projection
$\Im(d_{\min})=\Fix(d_{\min})$. Therefore 
\begin{equation}
P(\mL)=\Fix(d_{\min})=\Im(d_{\min}).
\end{equation}
We see that  $P(\mL)$ is the $(\min,+)$ (tropical) span of the columns of the matrix $d$. Analogously, if we denote by $d^t$ the transpose of $d$, we see that 
$\widehat{P}(\mL)=\Im(d^t_{\min})=\Fix(d^t_{\min})$ and therefore
$\widehat{P}(\mL)$ is the $(\min,+)$ row span of $d$.
We let $(u\oplus v)_i:=\min\{u_i,v_i\}$ and 
$(\lambda \odot v)_i:=\lambda+v_i$. Then 
for $x\in P(\mL)$ we have 
\begin{equation}
x=\oplus_j x_j\odot d(-,a_j)=\oplus_jD(Y(a_j),x)\odot d(-,a_j)=\oplus_jD(Y(a_j),x)\odot Y(a_j).
\end{equation}
and an analogous formula holds for $z\in \widehat{P}(\mL)=\Im(d^t_{\min})$.
From these we get that (\Cref{prop-11})
\begin{equation}
Y(a_k)=d(-,a_k)=\oplus_{a_j\leq a_k}d_{j,k}\odot Y(a_j)
\end{equation}
and 
\begin{equation}
\widehat{Y}(a_k)=d(a_k,-)=\oplus_{a_k\leq a_l}d_{k,l}\odot \widehat{Y}(a_l).
\end{equation}
These are the systems of equations we referred to in \Cref{re semantic}.

In section 5 we study how $P(\mL)$ changes when we enlarge the language corpus $\mL$. We prove that 
if a probabilistic language model $(\mL_1,d_1)$ is extended to $(\mL_2,d_2)$, namely if there is an isometric embedding $\phi:(\mL_1,d_1) \hookrightarrow (\mL_2,d_2)$ then there is 
an isometric embedding $\widetilde{\phi}:(P(\mL_1)),D_1) \hookrightarrow (P(\mL_2),D_2)$ such that
    $\widetilde{\phi}(Y_1(a))=Y_2(\phi(a))$. 
    Moreover there is a non-expansive, $(\min,+)$ projection $\mR:P(\mL_2) \to P(\mL_2)$ such that $\Im(\mR)=\widetilde{\phi}(P(\mL_1))$.  %

Using this we show that  if $\mL_1:=\{w_1,\dots, w_l\}$ is the set of words in the language and $b$ is a text in $\mL$ then 
$\mR(Y(b))=\bigoplus_{w_i\leq b} d_2(w_i,b)\odot Y_2(w_i)$
Introducing a temperature parameter $T$ we get
\begin{equation}
\mR(Y(b))=\lim_{T\to 0}-T \log(\sum_{w_i\leq b}
e^{-\frac{d(w_i,b)}{T}} e^{-\frac{Y(w_i)}{T}})
\end{equation}
Therefore for small $T$ we have
$e^{-\frac{\mR(Y(b))}{T}}\approx \sum_{w_i\leq b}
e^{-\frac{d(w_i,b)}{T}} e^{-\frac{Y(w_i)}{T}}.$

Putting  $v_i\coloneqq e^{-Y( w_i)}$  we have

\begin{equation}
 e^{-\frac{\mR(Y(b))}{T}} \approx \sum_i e^{-d(Y(w_i),b)/T} v_i
\end{equation}

This  approximation of text vectors is
similar to the one calculated by transformer neural networks in the self attention module.

In section~\ref{section 6} we describe a duality between semantics via texts extensions and text restriction. Indeed we have already seen that $P(\mL)=\{x|d_{\min}x=x\}$ is the $(\min,+)$ column span of $d$ and $\widehat{P}(\mL)=\{x|d^t_{\min}x=x\}$  is the $(\min,+)$ row span of $d$. It is easy to see that $d_{\min}x=x\iff d^t_{\min}(-x)=-x$ (\Cref{prop-antiisom}). Indeed it follows from the fact that $x_i\leq d_{i,j}+x_j \iff -x_j\leq d_{i,j} -x_i$. However this requires extending the $(\min,+)$ semifield, as well as the values of the directed metric, to $[-\infty,\infty]$. This results in the definition of extended $(\min,+)$ modules $P^-(\mL)$ and $\widehat{P}^-(\mL)$.

We show  that $(P^-(\mL),D)$ and $(\widehat{P}^-(\mL),D^t)$ are isometric (and in fact tropically anti-isomorphic).  We interpret this as saying that \textit{the semantic space $\widehat{P}^-(\mL)$, defined by extensions of texts is isomorphic to the semantic space $P^-(\mL)$ defined by restrictions of texts.}
We note that this is quite a non-trivial isomorphism as 
$P(\mL)$ and $\widehat{P}(\mL)$ don't even have the same number of extremal rays in general.
We give an example (Example 1) illustrating the polyhedra $P(\mL)$ and $\widehat{P}(\mL)$. 
In fact we consider the corresponding polyhedral cones $Q(\mL)$ and $\widehat{Q}(\mL)$.
We then define 
\begin{equation}
Q_0(\mL):=Q(\mL)\cap \Delta
\end{equation}
to be the intersection of $Q(\mL)$ with the unit simplex $\Delta$. Points in $Q_0(\mL)$ are normalized to probability distributions. Analogously define $\widehat{Q}_0(\mL)$) to be the intersection of $\widehat{Q}(\mL)$ with the unit simplex. 
Extremal rays of $Q(\mL)$ and $\widehat{Q}(\mL)$) define vertices of $Q_0(\mL)$ and $\widehat{Q}_0(\mL)$ respectively. 
We show the correspondence of extremal rays to lower sets and upper sets. The example also showcases the difference between a probabilistic language model and a general directed metric space where infinite distances are approximated uniformly by a big number $M$. 

Section 7 further explores the extremal rays of $P(\mL)$.

In section 8 we explore the relation with the so called Isbell completions $I(\mL)$ and $\widehat{I}(\mL)$. This is similar to the duality of section~\ref{section 6}.  The  Isbell adjunction is defined over the extended ring $[-\infty,\infty]$ and the fixed parts of the adjunction turn out to be $I(\mL)=\Im(d_{\max})$ where $d_{\max}(x)_i:=\max_j\{d_{i,j}+x_j\}$ and $\widehat{I}(\mL)=\Im(d^t_{\max})$.  In  (\Cref{prop-lattice-completion}) it is shown that $P(\mL)$ is the lattice completion of the so called Isbell completion.

When restricting coefficients to $[0,\infty]$, the Isbell completion has been studied by Willerton \cite{willerton2013tight} where it is proven to be isomorphic to the directed tight span $DTS(\mL)$ of 
Hirai and Koichi \cite{Hirai2012} (inspired
by the undirected tight span defined by Isbell~\cite{Isbell1964} and
Dress~\cite{Dress1984}).
It also generalizes the Dedekind-MacNeille completion of a poset.

Informally speaking $I(\mL)$ can be though of as the minimal space in which we can isometrically embed $\mL$. Unlike $P(\mL)$ though, it is far from being convex.
A simple example is shown in section 8.

Section 9 is a collection of observations about 
probabilistic language models and their relation to transformers.

As mentioned already, all constructions and results in this paper have categorical interpretations and though we have avoided using categorical language in the main text we explain briefly in Appendix A, for the benefit of readers familiar with categories,  these categorical interpretations.

Finally in Appendix B we present a general perspective which locates the language syntax and semantics problems  in the realm of a basic duality in mathematics which in its simplest form appears as a duality between algebra and geometry. This allows us to locate future directions of research. 

Experimental evidence for the semantic meaning of the (co-)Yoneda embedding vectors $\widehat{Y}(a_k)$  has been provided in \cite{liu2023meaning}. There experiments based on (a slight variation of) the co-Yoneda embedding vectors, were performed  using an actual Transformer LLM, by sampling over continuations of texts. 
Several semantic tests were conducted and the results were in general very good.

\subsection{Acknowledgements}
YV would like to thank Tai-Danae Bradley, Michael Douglas, Ioannis Emiris, Harris Papageorgiou,  Alex Takeda, John Terilla, Matthew Trager, Maxim Kontsevich, Matilde Marcolli, Jack Morava, Stefano Soatto,  and Elias Zafiris for useful conversations. He also would like to thank Anna Geneveaux for computing several useful examples of polyhedra for probabilistic language models during her internship. SG and YV thank Gleb Koshevoy and Panayotis Mertikopoulos for useful conversations. Finally YV would like to thank IHES for providing excellent working conditions.

\section{From probabilities of text extensions to  distances}
Consider a language with a set of words
$W:=\{w_1,\dots, w_l\}$. 
Consider also a set of training texts from the language, $\mL\coloneqq  \{a_0,a_1,\dots, a_n\}$ where
$a_i:=w_{i_1}\dots w_{i_{k_i}}$.
We endow $\mL$ with a poset structure where $a_i \leq a_j$ if and only if $a_i$ is a subtext of $a_j$.
We consider two possibilities for the notion of subtext. 
The first is  
\begin{equation}
a_i \leq_1 a_j \iff \exists a_k \in\mL \text{ such that } a_j=a_i a_k
\end{equation}
and we refer to this as the \textit{one sided subtext order} 
and the second is 
\begin{equation}
a_i\leq_2 a_j \iff \exists a_{k_1},a_{k_2} \in\mL \text{ such that }a_j=a_{k_1}a_ia_{k_2}
\end{equation}
and we refer to that as the \textit{two sided subtext order.}

We define always $a_0$ to be the empty text and  
$a_0$ is the only text such that $a_0\leq a_i \forall i$ in either order. (However 
see remark (1) bellow for how $a_0$ interacts with the probabilities we will soon add to the model.) 

If $a_i\leq_1 a_j$ in the one sided subtext order then $a_i\leq_2 a_j$.  The results and constructions that follow hold equally for both orders so we will simply write $a_i\leq a_j$ and when there is need to separate the two orders we will make a special comment.

\textit{If $a_i\leq a_j$ then denote by $\Prob(a_j|a_i)$ the probability of extension from $a_i$ to $a_j$.}

It is important to note that these probabilities are not calculated from a corpus of texts, as any probability for a sufficiently long text would be vanishingly small. Instead we are talking about the probabilities that the large language model (LLM) computes. Namely prompted with a text $a_i$ the model outputs a probability distribution $\Prob(a_iw_{j_1}|a_i) \forall w_{j_1} \in W$ and this is the probabilities we are referring to, above. To continue extending to 
$a_iw_{j_1} w_{j_2}$ we simply have
\begin{equation}
\label{eq conditional}
\Prob(a_iw_{j_1} w_{j_2}|a_i)=
\Prob(a_iw_{j_1}|a_i) \Prob(a_iw_{j_1} w_{j_2}|a_iw_{j_1})
\end{equation}
And continuing this way $\Prob(a_j|a_i)$ is computed. If $a_i$ is not a subtext of $a_j$ then we put $\Prob(a_j|a_i)=0$.

Recall that the LLM is trained to produce the probability distribution of the next word to a text, in such a way that the most likely next word predicted by the model is the one in the training text.

\textit{As a consequence of \Cref{eq conditional}  we make our fundamental assumption that}
\begin{equation}
\label{eq main assumption}
    a_i\leq a_j \leq a_k \implies 
    \Prob(a_k|a_i)=\Prob(a_k|a_j)\Prob(a_j|a_i)
\end{equation}

Note that the transformer  LLM produces the probabilities for the one-sided subtext order. However in the attention layers of the transformer,  two sided extensions are used in order to construct the text vector.  We consider therefore the case of the two sided order as well. Indeed in section 5 we will see that the text vector we define, is expressed in terms of word vectors when we consider the two sided subtext order. 
We put these together in the following 
\begin{definition}
A  probabilistic language model is a triple 
$(\mL,\leq,Pr)$ where, 
$\mL:=\{a_0,a_1,\dots,a_n\}$ is a collection of texts, $\leq$ is the subtext order and 
$\Pr:\mL\times \mL \to [0,1]$ is a function such that 
$a_i\leq a_j \leq a_k \implies 
    \Prob(a_k|a_i)=\Prob(a_k|a_j)\Prob(a_j|a_i)$.
    \end{definition}

Recall now the following 
\begin{definition}
$(X,\delta)$ is called a directed metric space if $X$ is a set and $\delta:X\times X \to (-\infty,\infty]$ satisfies the triangle inequality 
\begin{equation}
    \delta(a,c)\leq \delta(a,b)+\delta(b,c)
\end{equation} 
for all $a,b,c \in X$
and 
$\delta(a,a)=0, \forall a \in X$
\end{definition}

Note that this generalises usual metrics in that we don't require  $\delta(a,b)=\delta(b,a)$, $\delta(a,b)=0$ does not necessarily imply $a=b$ and moreover we allow negative values. This definition of a directed metric, in the special case of  positive valued $\delta$, has appeared in \cite{willerton2013tight, Lawvere73} and is also known as a generalised metric or a pseudo quasi metric.  

\begin{remark}
\label{re extended metric}
We need the following technical specification: In three cases (proposition 3 and the duality theorems of sections 6 and 8) we will need to extend definition 2, of a directed metric to allow the value $-\infty$ so that we will have. 
 $\delta:X\times X \to [-\infty,\infty]$. 
 In that case the definition is the same but we need to specify that we use the convention that  $+\infty$ is the absorbing element so that $s+(+\infty) =+\infty$ for all $s$ and in particular 
 $-\infty+ \infty=+\infty$.
This will be needed in Proposition 2.
In section~\ref{section 6} we will explain further that this is the so called $(\min,+)$ convention, as this is the only one compatible with the structure of $(\min,+)$ semiring; there is also a dual $(\max,+)$ convention.
\end{remark}

We define now a directed metric space structure on the underlying poset of a probabilistic language model $(\mL,\leq,\Pr)$.
\begin{definition}
\label{def PLM}
    Given the probabilistic language model $(\mL,\leq, \Pr)$ where $\leq$ is the subtext order and $\Prob(a_j|a_i)$ are the probabilities of extension, define the directed metric $d:\mL\times\mL \to [0, \infty]$ by

    \begin{equation}
d(a_i,a_j) = \left\{
\begin{array}{ll}
  -\logPr(a_j|a_i) & \text{if } a_i \leq a_j,\\
  
  \infty & \text{if $a_i$ and $a_j$ are not comparable} .
\end{array}
\right.
\end{equation}
\end{definition}
 It is clear that $d(a_i,a_i)=0$. To verify
that $d$ is a directed metric we have the following:
\begin{proposition} \label{prop-1}
  The map $d$ satisfies the triangle inequality:
  \begin{equation}
d(a_i,a_k)\leq d(a_i,a_j)+  d(a_j,a_k) \enspace,\label{e-triangular}
\end{equation}
  and equality holds if and only if $a_i\leq a_j\leq a_k$
  or $a_i\not\leq a_k$.
\end{proposition}
\begin{proof}
  Indeed, if $ a_i\leq a_j \leq a_k$ then 
  $a_i \leq a_k$,
  and the equality holds in~\eqref{e-triangular},
  since by our main assumption (the standard property of conditional probabilities),
  $\Prob(a_k|a_i)=\Prob(a_k|a_j)\Prob(a_j|a_i)$.
  If $a_i\not\leq a_k$, then, $d(a_i,a_k)=\infty$, and either $a_i\not\leq a_j$ or $a_j\not\leq a_k$, which entails that both sides of~\eqref{e-triangular}
  are equal to infinity. Finally, if $a_i\leq a_k$ but $a_i\not\leq a_j$
  or $a_j\not\leq a_k$, the left-hand side of ~\eqref{e-triangular}
  is finite whereas the right-hand side is $+\infty$.
\end{proof}
We then have
\begin{corollary}\label{coro-1}
The following statements are equivalent:

\begin{enumerate}
\item 
    $a_i\leq a_j \leq a_k$
    \item 
    $d(a_i,a_k)= d(a_i,a_j)+d(a_j,a_k) 
    \text{ and }
    d(a_i,a_k)< \infty   $
    
\end{enumerate}

\end{corollary}

\begin{remark}
Note that from \Cref{prop-1} and \Cref{coro-1}  it follows that the partial order $\leq$ on $\mL$ can be fully recovered by the directed metric $d$ or equivalently the conditional probabilities $\Pr(a_j|a_i$), therefore we will also denote the probabilistic language model $(\mL,\leq, \Pr)$  as $(\mL,\Pr)$ or $(\mL,d)$.  \end{remark}

\begin{remark}
In a Large Language model 
probabilities are normalized to add up to one, over all extensions of a given text by a word.
\end{remark}

\begin{remark}
    Note that the probabilistic language model $(\mL,d)$ is a special case of a directed metric space. Whenever it is possible we will prove results for a general directed metric space and derive the language case as a corollary. Moreover, it is possible to imagine that even the main assumption $a_i\leq a_j \leq a_k \implies 
    \Prob(a_k|a_i)=\Prob(a_k|a_j)\Prob(a_j|a_i)$ should be generalized to  $\Prob(a_k|a_i)\geq \Prob(a_k|a_j)\Prob(a_j|a_i)$, namely this of a general directed metric space with $d(a_i,a_k)\leq  d(a_i,a_j)+d(a_j,a_k)$.
     This is a reasonable assumption by itself and can be interpreted a saying that the shortest path to go from $a_i$ to $a_k$ is at least as short as a path that is forced to go from $a_i$ to $a_k$ but passing though $a_j$. 
     As we will see in what follows, the only result that requires the main assumption of conditional probabilities multiplying is \Cref{th-2} and all the rest are valid for general directed metric spaces.
     It is a matter of experimental verification to check for a given LLM if the multiplicative assumption is best or the general case. 
\end{remark}
\begin{remark}
\label{re extended model}
Note that we can slightly modify the definition of the Probabilistic Language model so that instead of $\Pr$ taking values in $[0,1]$ we put
$\Pr: \mL \times \mL \to [0,\infty)$. 
Then definition 3 will again produce a directed metric space. In fact we develop most of the theory using the more general extended  assumption since most results are valid for general directed metric spaces as we mentioned in the previous remark. 
\end{remark}

\begin{remark}
\label{re grading}
Since the machine produces probabilities for all possible next words it is natural to assume it is learning probabilities of extension for the free monoid generated by words. 
Obviously most strings of words will have vanishing probability and only those which are part of the language should have big probability.

We can then consider $\mL$ to contain the whole free monoid and it is natural to grade it by the word length of each text. 
\end{remark}
\begin{remark}
Note that if we  assume that there exists  $a_0$ such that 
$a_0\leq a_k \forall a_k \in \mL$ then 
$a_i\leq a_j$ implies $a_0\leq a_i \leq a_j$ and therefore 
$d(a_0,a_j)= d(a_0,a_i)+d(a_i,a_j)$ and thus
$d(a_i,a_j)=d(a_0,a_j)-d(a_0,a_i)$. 
This is equivalent to the statement that there is a globally defined probability distribution for absolute probabilities of texts, giving rise to all the conditional probabilities. Namely if $a_i\leq a_j$ then 
$\Prob(a_j|a_i)=\frac{\Prob(a_j|a_0)}{\Prob(a_i|a_0)}$.
The element $a_0$ can be considered to be the empty text and from this point of view it is natural to assume it exists in $\mL$.  
However the fact that it implies all conditional probabilities come from a global probability distribution shows that the inclusion of $a_0$ in the probabilistic language model, is not an entirely  trivial assumption. 

It would be a matter for experimental verification to see if it applies in the transformer Large Language Models. 
Therefore we will not assume it by default. We will specify explicitly whenever we assume $a_0 \in \mL$. 
\end{remark}

Next, we illustrate what the main assumption implies by the following:
\begin{proposition}\label{prop-2}
Consider a probabilistic language model $(\mL,\leq,\Pr)$ then on every connected component $C$ of the Hasse diagram of $\mL$, there is a function $P_C:C \to [0,\infty)$ such that if $a_i,a_j \in C$ and $a_i\leq a_j$ then $\Pr(a_j|a_i)=\frac{P_C(a_j)}{P_C(a_j)}$. The function $P_C$ is unique up to multiplication by a positive number.
\end{proposition}

\begin{proof}
The fact that $(\mL,\leq,\Prob)$  is a probabilistic language model means that
\[%
a_i\leq a_j \leq a_k \text{ is equivalent to }
\Prob(a_k|a_i)=\Prob(a_k|a_j)\Prob(a_j|a_i) \text{ and } \Prob(a_k|a_i)<\infty.
\]%
Let $G$ denote the directed graph which is the Hasse diagram of $C$. We construct a new weighted graph $\tilde{G}$ as follows:
If $a_i \leq a_j$, we draw an arrow from node $a_i$ to node $a_j$
with weight $\Prob(a_j|a_i)$. If $a_j\leq a_i$, we draw an arrow
from node $a_i$ to node $a_j$ with weight $\Prob(a_j|a_i)^{-1}$.
We now choose arbitrarily an element
$c \in C$. If $a_i \in C$, we define $P_C(a_i)$ to be the weight in the graph $\tilde{G}$ of an arbitrary path from the point $c$ to $a_i$. Owing to our main assumption,
\eqref{e-main-assum}, the weight is independent of the choice of the path from $c$ to $a_i$. Moreover, for all $i,j$ such that $a_i\leq a_j$, we have
$P_C(a_i) \Pr(a_j|a_i) P_C(a_j)^{-1}=1$ therefore
$\Pr(a_j|a_i)=\frac{P_C(a_j)}{P_C(a_i)}$.

Picking a different reference element $c'\in C$ scales $P_C(a_i)$ by $P_C(c')$ therefore the ratio stays the same. 
\end{proof}

\section{From the text metric space $\mL$ to the polyhedra $P(\mL)$ and $Q(\mL)$}\label{section-3}
First notice that we can also equip $\mL$ with the transpose directed metric $d^t$ where $d^t(a_i,a_j):=d(a_j,a_i)$.

We now construct two directed metric, polyhedra $P(\mL)$ and $\widehat{P}(\mL)$ in which the directed metric space $(\mL,d)$ is isometrically embedded as a special set of extremal rays.

To that end, we equip
$\{\R\cup \{\infty\}\}^n  \text{\textbackslash}\{(\infty,\dots,\infty)\}$ for $n\geq 2$, with the {\em Funk} metric $D$ defined
by 
\begin{equation}
  D(x,y) \coloneqq \inf \set{\lambda\in \R\cup\{+\infty\}}{\lambda + x\geq y} = \max_i\set{y_i-x_i}{x_i\neq \infty} \enspace .
  \label{e-def-funk}
\end{equation}

This is a directed metric. Note that it takes possibly negative values, and that
it can also take the  value $\infty$.

We also denote by $D^t$ the transpose directed metric with $D^{t}(x,y):=D(y,x)$.

\begin{definition}
\label{def P}
Let $(P(\mL),D)$ be the directed metric polyhedron
\begin{equation}
P(\mL):=\{x=(x_1,\dots,x_n) \in \{\R\cup \{\infty\}\}^n \text{\textbackslash}\{(\infty,\dots,\infty)\}| x_i \leq x_j + d_{i,j}\} .
\end{equation}
Moreover  let $(\widehat{P}(\mL),D^t)$ 
be the directed metric polyhedron
\begin{equation}
\widehat{P}(\mL):=\{y=(y_1,\dots,y_n) \in \{\R\cup \{\infty\}\}^n \text{\textbackslash}\{(\infty,\dots,\infty)\}| y_i \leq y_j + d_{j,i}\}.
\end{equation}
\end{definition}

For the following proposition, we need to extend the Funk metric in
eq.~\ref{e-def-funk}
to $n=1$.
For this we will use the fact that $\max \emptyset=-\infty$. From this it follows that for $n=1$ the Funk metric
$d_\R: (-\infty,\infty]^2 \to [-\infty,\infty]$ is given by
\begin{equation}
d_\R(s,t):=t-s \text{ if } s\neq \infty \text{ and }d_\R(\infty,t)=-\infty \enspace.
\end{equation}
In particular  $d_\R(\infty,\infty)=-\infty$.
Notice that $d_\R$ can also take the value $-\infty$ and this case of a directed metric, was explained in \Cref{re extended metric}. %

\begin{proposition}\label{prop-3}
$P(\mL)$ is the set of  
non expansive maps $x:(\mL,d^t) \to ((-\infty,\infty],d_\R)$. 
Namely $x$ satisfies
\begin{equation}
d_\R(x(a_j),x(a_i))\leq d^t(a_j,a_i) 
\end{equation}

Moreover    $\widehat{P}(\mL)$ is the set of non expansive maps $y:(\mL,d) \to ((-\infty,\infty],d_\R)$.
Namely $y$ satisfies
\begin{equation}
d_F(y(a_j),y(a_i))\leq d(a_j,a_i) 
 \end{equation}

 \end{proposition}
 \begin{proof}
To see this description of $P(\mL)$, let $x_i\coloneqq x(a_i)$.
 Then $$d_\R(x(a_j),x(a_i))\leq d^t(a_j,a_i)   \iff x(a_i) -x(a_j)\leq d(a_i,a_j) \iff x_i -x_j\leq d_{i,j} .$$

 Likewise to see this description of $\widehat{P}(\mL)$, let $y_i:=y(a_i)$.
 Then $$d_\R(y(a_j),y(a_i))\leq d(a_j,a_i)   \iff y(a_i) -y(a_j)\leq d(a_j,a_i) \iff y_i -y_j\leq d_{j,i} .$$

 \end{proof}
 \begin{remark}
  Following \Cref{prop-3}  we see that we can  view  $P(\mL)$ as a space of functions on the metric space $\mL$ and we will see in Section 4 that it is similar to considering real vectors as real valued functions on a set. 
  \end{remark}

\begin{proposition}\label{prop-yoneda-alt}\label{prop-4}
  The map 
  \begin{equation}
  Y:(\mL,d)  \hookrightarrow (P(\mL),D) \text{ given by } 
  Y(a_k)\coloneqq  d(-,a_k)
  \end{equation}
  is called the \emph{Yoneda embedding}
  \footnote{The reason for the name Yoneda embedding comes from its appearance in category theory and was explained in \cite {BTV2021}. It is similar to the so called, Kuratowski embedding of a metric space.}
  and is
  an isometric embedding.
  Moreover the map  
  \begin{equation}
  \widehat{Y}:(\mL,d) \hookrightarrow  (\widehat{P}(\mL),D^t) \text{ given by }
  \widehat{Y}(a_k)\coloneqq  d(a_k,-)
  \end{equation}
  is also
  an isometric embedding and is called the \emph{co-Yoneda embedding}.
\end{proposition}
\begin{proof}

First note that for any $a_k \in \mL$ the  function $Y(a_k):=d(-,a_k)$ is in $P(\mL)$ and the function $\widehat{Y}:=d(a_k,-):\mL \to [0,\infty]$ is in $\widehat{P}(\mL)$.

Indeed by the triangle inequality, $d(a_i,a_k)\leq d(a_i,a_j)+d(a_j,a_k)$, in other words if $x:=d(-,a_k)$ and $x_i:=d(a_i,a_k)$ then $x_i \leq x_j +d_{i,j}$ proving that $Y(a_k) \in P(\mL)$.

Analogously $d(a_k,a_i)\leq d(a_k,aj)+d(a_j,a_i)$ and therefore  if $y:=d(a_k,-)$ then $y_i \leq y_j +d_{j,i}$ proving that $\widehat{Y}(a_k) \in \widehat{P}(\mL)$.

  Moreover, the inequality ~\Cref{e-triangular}  entails that $d(-,a_i) + d(a_i,a_j)\geq d(-,a_j)$,
  and so, $D(d(-,a_i),d(-,a_j))\leq d(a_i,a_j)$.  On the other hand, if $x,y \in P(\mL)$ then  $D(x,y)\geq y_l-x_l$ for all $l$ and thus
  $D(d(-,a_i),d(-,a_j)\geq d(a_i,a_j)-d(a_i,a_i)=d(a_i,a_j)$. Consequently $D(d(-,a_i),d(-,a_j))=d(a_i,a_j)$
  i.e. $a\mapsto d(-,a)$ is an isometry.
   
  Likewise we have $d(a_j,a_i)+d(a_i,-)\geq d(a_j,-)$ and $d(a_j,-)-d(a_i,-)\leq d_{j,i}$ which implies that 
  $ D(d(a_i,-),d(a_j,-))\leq d_{j,i}$. 
 
 Moreover 
 $D(d(a_i,-),d(a_j,-))\geq d(a_j,a_i)-d(a_i,a_i)=d(a_j,a_i)$.
 Thus 
 
 $D(d(a_i,-),d(a_j,-))=d_{j,i}$ .
  
  \end{proof}

 To further understand the polyhedron $P(\mL)$  we consider the change of variables 
$z_i\coloneqq e^{-x_i}$ and  introduce the following:
\begin{definition}
\label{def Q}
Let $Q(\mL)$ be the polyhedral cone
\begin{equation}
Q(\mL):=\{z=(z_1,\dots z_n) \in [0,\infty)^n \text{\textbackslash}\{(0,\dots,0)\} |z_i\geq \Prob(a_j|a_i)z_j \}
\end{equation}
Moreover let $\widehat{Q}(\mL)$ be the polyhedral cone
\begin{equation}
\widehat{Q}(\mL):=\{u=(u_1,\dots u_n) \in [0,\infty)^n \text{\textbackslash}\{(0,\dots,0)\}|u_i\geq \Prob(a_i|a_j)u_j \}
\end{equation}
\end{definition}
Note that if $z$ is in $Q(\mL)$ then $\lambda z \in Q(\mL)$ for $\lambda \in [0,\infty)$ therefore $Q(\mL)$ is indeed a polyhedral cone in the positive orthant and so is $\widehat{Q}(\mL)$.

\textit{To simplify notation we introduce the convention  that if $v=(v_1,\dots, v_n) \in \R^n$ then}
\begin{equation}
e^v:=(e^{v_1},\dots,e^{v_n})
\text{ and }
\log(v):=(\log(v_1),\dots, \log(v_n))
\end{equation}

 We see that 
\begin{equation}
Q(\mL)=\{z\in [0,\infty)^n| z:=e^{-x} \text { for } x \in P(\mL)\}
\end{equation}

and 
\begin{equation}
\widehat{Q}(\mL)=\{u\in [0,\infty)^n| u:=e^{-y} \text { for } y \in \widehat{P}(\mL)\}
\end{equation}

And vice versa 
\begin{equation}
P(\mL)=\{x\in (-\infty,\infty]^n| x:=-\log(z) \text { for } z= \in Q(\mL)\} 
\end{equation}
and 
\begin{equation}
\widehat{P}(\mL)=\{y\in (-\infty,\infty]^n| y:=-\log(u) \text { for } u \in \widehat{Q}(\mL)\}.
\end{equation}

Using the map $-\log: Q(\mL)\to P(\mL)$ we can define a directed metric $D_Q$ on $Q(\mL)$ using the Funk metric $D$ on $P(\mL)$.  We put
\begin{equation}
\label{eq D_Q}
 D_Q(z,z'):=\max_i\{\log(\frac{z_i}{z'_i})|z'_i\neq 0\}.
 \end{equation}

By definition  we have 
 \begin{equation}
 D_Q(z,z')=D(-\log z,-\log z') \text{ and }
 D(x,x')=D_Q(e^{-x},e^{-x'}).
 \end{equation}
 
 Clearly the transpose $D^t_Q$ defines a directed metric on $\widehat{Q}(\mL)$.

Then \Cref{prop-yoneda-alt} implies that 
\begin{corollary} 
\label{co Yoneda Q}
The maps 
\begin{equation} 
    e^{-Y}:(\mL,d) \to (Q(\mL),D_Q) \text{ and }
    e^{-\widehat{Y}}:(\mL,d) \to (\widehat{Q}(\mL),D^t_Q)
\end{equation}
    are isometric embeddings
\end{corollary}
\begin{proof}
It follows  from \Cref{prop-2} since
\[ d(a_k,a_l)=D(Y(a_k),Y(a_l)) =D_Q(e^{-Y(a_k)},e^{-Y(a_l)})
\enspace .\qedhere
\]
\end{proof}

We define the \emph{unit simplex}
 \[ \Delta:=\{z\in [0,1]^n |\sum_i z_i=1\}
 \enspace .
 \]
\begin{definition}
\label{def Q_0}
Define the polyhedron $Q_0(\mL)$ by
 \begin{equation}
 \label{eq Q_0}
 Q_0(\mL):=Q(\mL)\cap \Delta.
 \end{equation}
Then $(Q_0(\mL),D_Q)$ is a directed metric polyhedron and points in $Q_0(\mL)$ are probability distributions.

Analogously we define the polyhedron $\widehat{Q}_0(\mL)$ as the intersection of $\widehat{Q}(\mL)$ with the unit simplex and $D^t_Q$ is a directed metric on it.
\end{definition}
\begin{remark}
\label{re normalization}
The polyhedra $Q_0(\mL)$ and $\widehat{Q}_0(\mL)$ define a normalization to probability distributions of our probabilistic language model $\mL$. Indeed in the definition of $(\mL,\leq,\Pr)$ we only ask for the conditional probabilities multiplicative property but there is no normalization to a probability distribution. 

Now if we consider the vertex of $Q_0(\mL)$, corresponding to the ray  generated by $e^{-Y(a_k)}$, it  will be
$\frac{1}{n(a_k)}e^{-Y(a_k)}$ where 
$$n(a_k):=\sum_{a_j\leq a_k}(e^{-Y(a_k)})_j=\sum_{a_j\leq a_k}\Pr(a_k|a_j)$$ is the normalization factor.

While the vertex of $\widehat{Q}_0(\mL)$corresponding to the ray  generated by $e^{-\widehat{Y}(a_k)}$  will be
$\frac{1}{\widehat{n}(a_k)}e^{-Y(a_k)}$ where 
$$\widehat{n}(a_k):=\sum_{a_k\leq a_l}(e^{-\widehat{Y}(a_k)})_j=\sum_{a_k\leq a_l}\Pr(a_l|a_k)$$ is the normalization factor.
\end{remark}

\begin{remark}
  The polyhedral cone $Q(\mL)$ is a generalization of the \emph{order polytope}
  defined by Stanley~\cite{Stanley86}. The order polytope corresponds to the case where 
  $\Prob(a_j|a_i)$ takes only the values $0$ or $1$, moreover, Stanley
  adds the ``box constraint''
  $z_i\in [0,1]$ which translates to $x_i\in [0,\infty]$.
  Up to the box constraint, $Q(\mL)$ corresponds to
  the order polytope of the poset 
$(\mL^{\op},\leq)$ where $\mL^{\op}$ is the opposite poset.

Stanley \cite{Stanley86} proves that  vertices of an order polytope correspond to upper sets of the poset.
We will prove a generalization of that result in \Cref{th-2} in section \ref{section 3.2}.
\end{remark}

We now explain what is the geometric  meaning of the coordinates of a point $x=(x_1,x_2,\dots,x_n)\in P(\mL)$ and a point
$y=(y_1,y_2,\dots,y_n)\in \widehat{P}(\mL)$.
\begin{proposition}\label{prop-5}
If $x \in P(\mL)$ then 
\begin{equation}
x_i=D(d(-,a_i),x)=D(Y(a_i),x).
\end{equation}
Moreover if $y \in \widehat{P}(\mL)$ then 
\begin{equation}
y_i=D^t(y,d(a_i,-))=D(\widehat{Y}(a_i),y).
\end{equation}
\end{proposition}

\begin{proof}
We have $D(d(-,a_i),x)) \geq x_i-d(a_i,a_i)=x_i$.

On the other hand $D(d(-,a_i),x))=\max_j\{x_j-d_{j,i}\}\leq x_i$ where the last inequality follows because
$x \in P(\mL) \iff x_j \leq x_i +d_{j,i}$ which implies $x_j-d_{j,i} \leq x_i$.

Moreover if $y \in \widehat{P}(\mL)$ then $D^t(y,d(a_i,-)) \geq y_i-d(a_i,a_i)=y_i$

On the other hand

$D^t(y,d(a_i,-))=\max_j\{y_j-d_{i,j}\} \leq y_i$
since
$y \in \widehat{P}(\mL) \iff y_j \leq y_i +d_{i,j}$ which implies $y_j-d_{i,j} \leq y_i$.
\end{proof}

\begin{remark}
\begin{enumerate}
    \item
We note that using the previous proposition,
the defining inequalities, $x_i\leq x_j+d_{i,j}$ of $P(\mL)$ become   
\begin{equation}\label{eq-45}
D(Y(a_i),x)\leq D(Y(a_i),Y(a_j))+D(Y(a_j),x).
\end{equation}
 Namely they are triangle inequalities for maps $x:\mL \to ((-\infty,\infty],d_\R)$ and for the maps
 $Y(a_k)=d(-,a_k):\mL \to ((-\infty,\infty],d_\R)$.
  Therefore we can think of $P(\mL)$ as the space of all maps $x=\mL \to ((-\infty,\infty],d_\R)$ that satisfy the triangle inequalities for the metric $D$, with respect to all the maps $Y(a_k)=d(-,a_k):\mL \to ((-\infty,\infty],d_\R)$. Thus $P(\mL)$ is a kind of convex metric span of the maps $Y(a_k)=d(-,a_k): \mL \to ((-\infty,\infty],d_\R)$.

 \item
 Analogously, the defining inequalities, $y_i\leq y_j +d_{j,i}$ of $\widehat{P}(\mL)$ become
 \begin{equation}
 D^t(y,\widehat{Y}(a_i)) \leq D^t(y,\widehat{Y}(a_j)) + D^t(\widehat{Y}(a_j),\widehat{Y}(a_i)),\label{eq-46}
  \end{equation}
 namely the triangle inequalities for maps $y:\mL\to ((-\infty,\infty],d_\R)$. This implies that $\widehat{P}(L)$  is the space of all maps
 $y=\mL \to ((-\infty,\infty],d_\R)$ that satisfy the triangle inequalities for the metric $D^t$, with respect to all the maps $\widehat{Y}(a_k)=d(a_k,-):\mL \to ((-\infty,\infty],d_\R)$. Again we see $\widehat{P}(L)$ 
 as a kind of convex metric span of $\widehat{Y}(a_k)=d(a_k,-)$
 
 \item 
A restatement of \eqref{eq-45} is to say that the shortest path that connects $Y(a_i)$  and $x$ is at most as long as the shortest path that connects them but has to also go though $a_j$.
Analogously for \eqref{eq-46}.
\item{Note that all constructions and results in this section work for a general directed metric space and not just for the special case of a probabilistic Language model}
 \end{enumerate}
 \end{remark}

\subsection{Texts define special Extremal rays of $P(\mL)$ and $Q(\mL)$}
\begin{definition}
    An extremal ray of a polyhedral cone in $\R^n$ is a ray generated by a vector that cannot be expressed as a positive linear combination of two non-proportional vectors in the polyhedral cone. 
\end{definition}
Recall that a vector in a polyhedral cone in $\R^n$ generates an extremal ray if and only if it saturates $n-1$ linearly independent inequalities~\cite{schrijver}.
\begin{definition}
An additive extremal ray of $P(\mL)$ (respectively $\widehat{P}(\mL)$) is defined to be the image under $-\log$ of a usual extremal ray of the polyhedral cone $Q(\mL)$ (respectively $\widehat{Q}(\mL)$).
\end{definition}

Note that the name additive extremal ray is chosen since a usual extremal ray in $Q(\mL)$ is invariant under scaling by $\lambda$ and therefore its image under $-\log$ is invariant under translation by $-\log \lambda$. Note that the extremal rays of $Q(\mL)$ and of $\widehat{Q}(\mL)$ have generators which have in general some zero coordinates. Then, their $-\log$-images have vectors such that some of their coordinates are $\infty$, which is why we speak of ``additive extremal rays'' of $P(\mL)$ and $\widehat{P}(L)$ (they are not extremal rays in the usual sense).

\textit{From now on though we will simply refer to the additive extremal rays of $P(\mL)$ as \emph{extremal rays}, when there is no chance of confusion.}

We now define a directed graph associated to any point $x \in P(\mL)$ which encodes the saturated inequalities satisfied by the point $x$ and which we call its \textit{saturation graph} $S(x)$.

\begin{definition}\label{def saturation graph}
Let $x\in P(\mL)$. Define $S(x)$, the {\em saturation graph} of $x$, to be the graph whose vertices are the elements of $\mL$ and whose set of directed edges $E(x)$ is the set of saturated inequalities that coordinates of $x$ satisfy, namely
$E(x)\coloneqq \{(a_i,a_j): x_i=x_j+d_{i,j} \}$. When $(a_i,a_j) \in E(x)$ we introduce a directed edge from $a_i$ to $a_j$. 
\end{definition}

The graph always contains trivial arcs $(a_i,a_i)$ (loops), for $i\in[n]$, since $d_{i,i}=0$. It contains non-trivial arcs if and only if $x$ is on the boundary of $P$. 

The graph $S(x)$ and in particular its support $\Supp(x)$ encodes all the hyperplanes on which $x$ lies.

Note that the graph can be disconnected.

\begin{theorem}\label{th-1}
The isometric embedding $Y: \mL \hookrightarrow  P(\mL)$, maps points of $\mL$ to extremal rays of the polyhedron $P(\mL)$ namely $Y(a_k)=d(-,a_k)$ is an extremal ray in $P(\mL)$.
Moreover the isometric embedding $\widehat{Y}: \mL \hookrightarrow  \widehat{P}(\mL)$, maps points of $\mL$ to extremal rays of the polyhedron $\widehat{P}(\mL)$ namely $\widehat{Y}(a_k)=d(a_k,-)$ is an extremal ray in $\widehat{P}(\mL)$.

\end{theorem}

\begin{proof}
We have $Y(a_k):=d(-,a_k)$, and therefore $Y(a_k)_i=d(a_i,a_k), i=1\dots |\mL|$.
 Define the support of $Y(a_k)$, $\Supp(Y(a_k))$, to be the set of texts $a_i$ such that $Y(a_k)_i$ is finite. We recall that a vector in a cone in $\R^n$ defined by finitely many linear constraints generates an extreme ray of the cone if, and only if, the family of gradients of active constraints at this point is
 of rank $n-1$.

 Let $x:= Y(a_k)$, and $y\in Q(\mL)$ denote the image of $x$ by the map which applies $\exp(-\cdot)$ entrywise. Each edge $(a_i,a_j)$ of the saturation graph $S(Y(a_k))$ yields
 $x_i = d(a_i,a_j) + x_j$, and so the vector $y$ induces the active inequality $y_i = \Prob(a_j|a_i) y_j$ with gradient $e_i - \Prob(a_j|a_i) e_j$ where $e_i$ denotes the $i$th vector of the canonical basis of $\R^n$. 
 Moreover, each text $a_i$ in $\mL\setminus \Supp(Y(a_k))$ yields the active inequality $y_i=0$, with gradient $e_i$.

The saturation graph $S(Y(a_k))$ has a connected component which is a directed tree with $a_k$ as its root since  $Y(a_k)_i=Y(a_k)_j+d_{i,j} \iff
d(a_i,a_k)=d(a_i,a_j)+d(a_j,a_k)$ and from corollary 1 it follows that  $a_i \leq a_j \leq a_k$, namely $a_j$ extends $a_i$ and $a_k$ extends $a_j$.
It has also trivial connected components, reduced
to loops at the vertices $a_i$ such that $a_i\not\in\Supp(Y(a_k))$.
Using the fact that the non-trivial connected component
of $S(Y(A_k)$ is a tree, we see that any vector $z$ satisfying the saturated equalities is uniquely defined by its value on the root of the tree. Hence, the space orthogonal to the family $e_i - \Prob(a_j|a_i) e_i$ with $(i,j) \in S(y_k)$ and $e_l$ with $a_l\in\mL\setminus \Supp(Y(a_k))$ is of dimension one, which entails that this family is of rank $|\mL|-1$, showing that  $y$ is an extreme ray of $Q(\mathcal{L})$.

Likewise let $\widehat{Y}(a_k):=d(a_k,-)$ and $\widehat{Y}(a_k)_i=d(a_k,a_i)$. If $a_k\leq a_i \leq a_j$ then $d(a_k,a_j)=d(a_k,a_i)+d(a_i,a_j)$, i.e. $\widehat{Y}(a_k)_j=\widehat{Y}(a_k)_i+d_{i,j}$.

The saturation graph for $\widehat{Y}(a_k)$  is the same as for $Y(a_k)$ with all arrow reversed.
Therefore the same proof applies. 
\end{proof}

\textit{It follows from \Cref{th-1}, that we can identify the  texts in $\mL$ with some of the extremal rays of $P(\mL)$ and also with some of the extremal rays in $\widehat{P}(\mL)$.}

However there are many other extremal rays of $P(\mL)$,
which we next characterize.
\subsection{All Extremal rays correspond to connected lower sets of $\mL$}
\label{section 3.2}
Consider the equations $y_i\geq \Prob(a_j|a_i)y_j$ which define $Q(\mL)$. We denote $P_{i,j}:=\Prob(a_j|a_i)$. 
If we assume  that $a_0\in \mL$ is the empty text, we have $a_0\leq a_i\leq a_j$ for any $a_i\leq a_j$, and then $\Prob(a_j|a_0)=\Prob(a_i|a_0)\Prob(a_j|a_i)$.
\begin{equation}
\text{ Define } P_i:=\Prob(a_i|a_0) \text{ then } 
P_{i,j}=\frac{P_j}{P_i}.\label{eq-47}
\end{equation}
Therefore $y_i \geq P_{i,j}y_j$ becomes 
$P_iy_i\geq P_j y_j$.
\begin{equation}
\text{ Define } \tilde{y}=(\tilde{y}_1,\dots \tilde{y}_n) \text{ where } \tilde{y_i}:=P_iy_i. 
\end{equation}

Notice that this change of coordinates maps extremal rays to extremal rays. 
We get then a new polyhedral cone 
\begin{equation}
\tilde{Q}(\mL):= \{\tilde{y}=(\tilde{y}_1,\dots,\tilde{y}_n)\in [0,\infty)^n \text{\textbackslash}\{(0,\dots,0)\}|\tilde{y_i}\geq \tilde{y_j} \text{ whenever }  a_i\leq a_j\}.
\end{equation}
Therefore
$\tilde{Q}(\mL):=\{\tilde{y}|y\in Q(\mL)\}$.

$\tilde{Q}(\mL)$ is a polyhedral cone variant of Stanley's order polytope for the opposite poset $\mL^{\op}$ -- the latter is the intersection of $\tilde{Q}(\mL)$ with the box $[0,1]^n$~\cite{Stanley86}.

The change of variables mapping $Q(\mL)$ to $\widetilde{Q}(\mL)$ can be also done under our more general assumption of a Probabilistic language model without assuming the existence of a global minimum $a_0 \in \mL$. Indeed

\begin{proposition}\label{prop-diagscaling}\label{prop-6}
Let $(\mL,\leq,\Prob)$ be a probabilistic language model then there is a diagonal change of variables mapping $Q(\mL)$ to $\widetilde{Q}(\mL)$.
\end{proposition}

\begin{proof}
The fact that $(\mL,\leq,\Prob)$  is a probabilistic language model means that
\[%
a_i\leq a_j \leq a_k \text{ is equivalent to }
\Prob(a_k|a_i)=\Prob(a_k|a_j)\Prob(a_j|a_i) \text{ and } \Prob(a_k|a_i)<\infty.
\]%
We define the directed graph $G$ whose nodes are the texts $a_1,\dots,a_n$.
If $a_i \leq a_j$, we draw an arrow from node $a_i$ to node $a_j$
with weight $\Prob(a_j|a_i)$. If $a_j\leq a_i$, we draw an arrow
from node $a_i$ to node $a_j$ with weight $\Prob(a_j|a_i)^{-1}$.
Consider the Hasse diagram of $\mL$. We make use of the observation
in~\Cref{prop-2}. For every connected component $C_m$,
let us select arbitrarily an element
$c_m$. If $a_i \in C_m$, we define $w_i$ to be the weight in the graph $G$ of an arbitrary path from the point $c_m$ to $a_i$. Owing to our main assumption,
\eqref{e-main-assum}, the weight is independent of the choice of the path from $c_m$ to $a_i$. Moreover, for all $i,j$ such that $a_i\leq a_j$, we have
$w_i\Prob(a_j|a_i)w^{-1}_{j}=1$. Setting $\tilde{y}_i = w_i y_i$,
we rewrite the constraint $y_i\geq \Prob(a_j|a_i)y_j$  as
$\tilde{y}_i\geq \tilde{y}_j$.

In this way, we transformed $Q(\mL)$ to $\tilde{Q}(\mL)$ by a diagonal scaling.
\end{proof}

We then have the following 
\begin{theorem}\label{prop-char-extreme}\label{th-2}
  The vector $\tilde{y}:=(\tilde{y}_1,\dots \tilde{y}_n) \in \tilde{Q}(\mL)$ generates an extremal ray of $\tilde{Q}(\mL)$ if and only if
the function $a_i \mapsto \tilde{y}(a_i):= y_i$ is a positive scalar multiple of the characteristic function of
    a lower set in $\mL$ whose Hasse diagram is connected.
\end{theorem}

\begin{proof}
Let $\{\lambda_1,\dots, \lambda_s\}$ be the distinct values taken by $\tilde{y}_i$,
ordered so that $0\leq \lambda_1< \lambda_2<\dots<\lambda_s$.
Let $\mL_m=\{i|\tilde{y}_i=\lambda_m\}$.

We define the rank of a family of affine inequalities to be the rank of the family of gradients of the affine forms defining these inequalities. For $\tilde{y}$ to be an extremal ray it has to saturate a family of inequalities of rank $n-1$, where 
$|\mL|=n$, see \cite{schrijver}.

Let us first assume $\lambda_1=0$.
The rank of the family of saturated constraints given by $\mL_1$ is then $r_1=|\mL_1|$ since we get equations of the form $\tilde{y}_i=0$ which is a hyperplane normal to $e_i$ for $i\in \mL_1$.

Moreover, we claim that the rank $r_k$ of the family of saturated constraints given by any $\mL_k$ for $k\neq 1$ given by  $r_k= |\mL_k|-c_k$ where $c_k$ is the number of connected components of the Hasse diagram of $(\mL_k,\leq)$. To see this,
it suffices to observe that a solution $h\in\mathbb{R}^{\mL_k}$ of the system of saturated inequalities
$h_i = h_j$ for $a_i \leq a_j$ and $a_i,a_j\in \mL_k$ is uniquely determined by fixing precisely one coordinate of $h$ on every connected component of the Hasse diagram (in other words, we have $c_k$ degrees of freedom for the choice of $h$).

Therefore the rank of the family of saturated constraints at the point $\tilde{y}$ is less than or equal to 
$|\mL_1|+|\mL_2|-1+\dots + |\mL_s|-1$.
We also have that 
$|\mL_1|+|\mL_2|+\dots + |\mL_s|=n$.
We know though that $\tilde{y}$ is an extremal ray if and only if  the rank of the family of saturated  constraints is $n-1$. 
Therefore we must have 
$n-1 \leq |\mL_1|+(|\mL_2|-1)+\dots + (|\mL_s|-1)=n-s+1$.

This is only possible if $s\leq 2$ but, for $\tilde{y}$ to generate an extremal ray, not all coordinates of $\tilde{y}$ can be $0$,
and then our assumption $\lambda_1=0$ excludes the case $s=1$. This entails that $s\geq 2$ and therefore $s=2$. We then have 
$n-1=|\mL_1|+|\mL_2|-1$. In that case $\tilde{y}_i=0$ for $i \in \mL_1$ and 
$\tilde{y}_j=\lambda_2$ for $j\in \mL_2$. We then scale $\tilde{y}$ by  $\frac{1}{\lambda_2}$ so as to get a representative vector of the same the ray with 
$\tilde{y}_i=0$ for $i \in \mL_1$ and 
$\tilde{y}_j=1$ for $j\in \mL_2$.
Therefore $\tilde{y}$ is the characteristic function of $\mL_2$.

Moreover  $\mL_2$ is a lower set. 
Indeed if $a_j \in \mL_2$ and $a_i\leq a_j$ then $a_i \in \mL_2$. This holds because $a_j \in \mL_2$ implies $\tilde{y}(a_j)=1$ and $a_i\leq a_j$ implies  $\tilde{y}(a_i)=\tilde{y}(a_j)=1$ therefore
$a_i \in \mL_2$.

If now $\lambda_1>0$ then $n-1 \leq (|\mL_1|-1)+(|\mL_2|-1)+\dots + (|\mL_s|-1)=n-s$ therefore 
$s\leq 1$ which implies $s=1$ and 
$|\mL_1|=n$. In that case we have a single extremal ray $\tilde{y}=(1,1,\dots,1)$ which is the characteristic function of the maximal lower set $\mL_1=\mL$.

Conversely let $C$ be a lower set in $\mL$ and let $\tilde{y}:\mL \to \{0,1\}$ be the characteristic function of $U$.  Consider $a_j\in U$ then $\tilde{y}(a_j)=1$. Now if $a_i\leq a_j$ then $a_i \in U$ and therefore  $\tilde{y}(a_i)=1$ which means $\tilde{y}(a_j)=\tilde{y}(a_i)$.

\end{proof}

\begin{remark} Note that if $\mL$ admits a bottom element $a_0$ then any lower set is connected since it must include $a_0$, and the Hasse diagram of $\mL$
  contains a path from $a_0$ to every element of $\mL$.
\end{remark}

Notice that Stanley in \cite{Stanley86} has proven that vertices of his order polytope correspond to upper sets. 
In contrast, rays of $Q(\mL)$ correspond
only to {\em connected} lower sets. Notwithstanding the order reversal,
there is a discrepancy which arises
because Stanley considers the intersection of $Q(\mL)$ with a box,
with creates additional vertices, not associated to rays of $Q(\mL)$.

\begin{remark}
Note that a vector $Y(a_k)\in P(\mL)$ corresponds to the principal lower set generated by $a_k$. We will therefore call the extremal rays generated by images of the Yoneda embedding, \textit{principal extremal rays}. 
\end{remark}

\begin{corollary}\label{co coord extremal}
  Assume $\mL$ includes the bottom element $a_0$ and recall from
  \eqref{eq-47} that $P_i:=\Prob(a_i|a_0)$.
If $C$ is a lower set, the 
    extremal ray corresponding to $C$ is generated by $y$ in $Q(\mL)$ with coordinates 

     \begin{equation}
y_i = \left\{
\begin{array}{ll}
  \frac{1}{P_i} & \text{if } a_i \in C,\\
  
  0 & \text{if $a_i$ not in $C$} .
\end{array}
\right.
\end{equation}
    
\end{corollary}

\begin{proof}
  It follows from \Cref{th-2}
  and the change of coordinates $\tilde{y_i}=P_iy_i$ in eq 53. 
\end{proof}

\begin{remark}
Notice that \Cref{co coord extremal} is consistent  with the coordinates of
    an extremal ray $y$ in $Q(\mL)$, corresponding to a text $a_j \in \mL$.
Indeed according to corollary 2, any element $y$ on the extremal ray has 
$y_i=\frac{\lambda}{P_i}$ for $\lambda \in [0,\infty)$.
We also have $y_j=\Prob(a_j|a_j)=1$, therefore $\lambda=P_j$.
This implies 
$y_i=\frac{P_j}{P_i}=\Prob(a_j|a_i)$

     \begin{equation}
y_i = \left\{
\begin{array}{ll}
  \Prob(a_j|a_i) & \text{if } a_i \leq a_j,\\
  
  0 & \text{if $a_i$ not a subtext of $a_j$} .
\end{array}
\right.
\end{equation}
Now we want a general version of the Corollary~\ref{co coord extremal}. 
\end{remark}

For any subset $C$ of $\mL$, selecting an element $c\in C$, for every element $a_i$ in the connected component of $C$ in the graph induced by the Hasse diagram of $\mL$, we denote by $w^c_i$ the weight of any path from $c$ to $a_i$ in the directed graph constructed in the proof of~\Cref{prop-diagscaling}.
\begin{proposition}\label{prop-7}
  Let $C$ be a connected lower set of the Hasse diagram of $(\mL,\leq)$.
  Let $c$ denote any element of $C$. Then, the vector
  \begin{equation}
    y_i=\frac{1}{w^c_i}, \qquad \text{for } a_i \in C, \qquad y_i = 0 \text{ for } a_i\in\mL\setminus\{C\}%
    \label{e-carac-extreme}
  \end{equation}
  generates an extreme ray of $Q(\mL)$, and all the extreme rays arise
  in this way.
\end{proposition}
\begin{proof}
  We showed in~\Cref{prop-char-extreme} that $\tilde{y}$ is a positive
  scalar multiple of the characteristic function of $C$.
  If $a_i$ belongs to $C$, we have $y_i = (w_i^{c})^{-1} \tilde{y_i}$, from which
  \eqref{e-carac-extreme} follows. We note that a change of the reference point $c$ in $C$ only modifies the vector $w^c$ by a positive scalar multiple.
  Indeed, for all $c$ and $c'\in C$, we have $w^c = \mu w^{c'}$ where
  $\mu$ is the weight of any path from $c$ to $c'$ in the directed graph $G$.
\end{proof}

\begin{proposition}   \label{prop-8}
  If $y$ generates an extremal ray of $Q(\mL)$ corresponding to a lower set $C$ in $\mL$ then
the saturation graph of $y$ has an edge from $a_i$ to $a_j$ if and only if $a_i\leq a_j$ for $a_i,a_j \in C$.
\end{proposition}

\begin{proof}
  This follows from~\eqref{e-carac-extreme}, using the main assumption~\eqref{e-main-assum}.
Indeed if $a_i\leq a_j$ in $C$ then we have 
 $y_i=\frac{1}{w^c_i}$ and  $y_j=\frac{1}{w^c_j}$. Therefore $y_i=\frac{w^c_j}{w_i}y_j$ and therefore 
$y_i=\Prob(a_j|a_i)y_j$ which means that there is an edge from $a_i$ to $a_j$ in the saturation graph of $y$.
  \end{proof}

\begin{corollary}
\label{cor 4}
Extremal rays of $\widehat{Q}(\mL)$ and $\widehat{P}(\mL)$
correspond to connected upper sets of $\mL$.
\end{corollary}

\begin{proof}
  We have $\widehat{Q}(\mL)=Q(\mL^{\op})$ and upper sets of $\mL$ correspond to lower sets of $\mL^{\op}$ therefore the result follows from \Cref{prop-7}.
\end{proof}
\begin{remark}
Note that \Cref{prop-diagscaling}, and \Cref{prop-char-extreme} 
are only valid for a probabilistic language model and not  for general directed metric space. In the latter case there will still be exponentially many extremal rays not coming from the Yoneda embedding, but we the characterization in terms of connected lower sets no longer holds.
\end{remark}

\begin{corollary}
If the empty text $a_0$ is in  $\mL$
then the set $P^0(\mL)$ of extremal rays of $P(\mL)$ is identified with the lower set completion of the poset $\mL$. 
\end{corollary}

\begin{proof}
    Since $a_0 \in \mL$, every lower set of $\mL$ is connected so $P^0(\mL)$ is identified with the set of lower sets of $\mL$. 

\end{proof}

\begin{remark}
  Note that having explicit equations for the polyhedral cone $Q(\mL)$,
  the extremal rays of $Q(\mL)$ can be computed, for instance
  by
  the double description method~\cite{FukudaProdon96}.
\end{remark}

\section{The polyhedron $P(\mL)$ as a $(\min,+)$  linear space}
\label{section-4}
To further understand the polyhedra $P(\mL)$ and $\widehat{P}(\mL)$ we need to consider their description in terms of tropical algebra.

Consider the metric space $(\mL,d)$.
Recall the $(\min,+)$ (tropical) semifield $\R_{\min}$ defined as 
$\R_{\min}:=((-\infty,\infty], \oplus_{\min},\odot)$ where for
 $s,t \in (-\infty,\infty]$,
 \begin{equation}
 \label{eq: min semiring}
s\oplus_{\min} t:=\min\{s,t\} \text{ and } 
s \odot t:=s+t.
\end{equation}
We denote by $d_{\min}: \mathbb{R}^n \to \mathbb{R}^n$ the $(\min,+)$ linear operator defined by
\begin{equation}
d_{\min}(x)_i:=\min_j\{d_{i,j}+x_j\} 
\label{e-def-dmin}
\end{equation}

\begin{proposition}
$(\mL,d)$ is a directed metric if and only if $d_{\min}^2=d_{\min}$, namely $d_{\min}$ is a $(\min,+)$ projector.
\end{proposition}

\begin{proof}
We have $d_{\min}^2=d_{\min} \iff d_{i,k}=\min_j\{d_{i,j}+d_{j,k}\}$
which is the same as the triangle inequality 
$d_{i,k}\leq d_{i,j}+d_{j,k} $.
\end{proof}
Let us denote $\Im(d_{\min})$, the image of $d_{\min}$, namely the $(\min,+)$ column span of $d$.

\begin{lemma}
We have $\Im(d_{\min})=\Fix(d_{\min})$, where $\Fix(d_{\min})$ is the $(\min,+)$ module $\Fix(d_{\min}):=\{x: d_{\min}(x)=x\}$.
\end{lemma}
\begin{proof}
It follows from $d_{\min}^2=d_{\min}$.
\end{proof}

We note now that there is a very natural description of our polyhedra as follows:
\begin{proposition}\label{prop-11}
The polyhedron $P(\mL)$ is equal to $\Im(d_{min})=\Fix(d_{\min})$ and the polyhedron $\widehat{P}(\mL)$ is equal to 
$\Im(d^t_{min})=\Fix(d^t_{\min})$.
\end{proposition}
\begin{proof}
Since $d_{\min}^2=d_{\min}$ we have that $x\in \Im(d_{min}) \iff d_{\min}x=x$ which means that 
$x_i=\min_j\{d_{i,j}+x_j\}$ and thus $x_i\le x_j +d_{i,j}$. Likewise for $\Im(d^t_{min})$ we get $x_i\le x_j +d_{j,i}$.
\end{proof}

Since we use much more often the $(\min,+)$ semifield than the $(\max,+)$ that will appear later on, \textit{to simplify notation we 
denote $\oplus_{\min}$ by $\oplus$}.

In particular we introduce the notation, for $u,v \in \mathbb{R}^n$,
$(u \oplus v)_i:=\min\{u_i,v_i\}$.
We then have
\begin{corollary}\label{coro-4}
If $x\in P(\mL)$ then 
\begin{equation}
x=\oplus_j D(Y(a_j),x) \odot Y(a_j).\label{eq-60}
\end{equation}
\end{corollary}

\begin{proof}
From \Cref{prop-11}, %
 $x\in P(\mL)=\Im(d_{\min})=\Fix(d_{\min}) \iff d_{\min}(x)=x $. Therefore
 we have the $(\min,+)$ linear expression for $x$ in terms of the columns of $d$:

$x=\oplus_j x_j \odot d(-,a_j)=\oplus_j x_j \odot Y(a_j)=\oplus_j D(Y(a_j),x) \odot Y(a_j)$.
\end{proof}

It is known that an order polytope, and more generally, an alcoved polytope (of $A_n$ type) is closed under $\min$ and $\max$, \eqref{eq-60} expresses this fact for our metric case for $\min$. In Proposition~\ref{prop-lattice-completion} we will see that $P(\mL)$ is also closed under $\max$.

\begin{proposition}
\label{prop linear system}
We have 
\begin{equation}
\label{eq system Yoneda}
Y(a_k)=\oplus_{a_j\leq a_k}d_{j,k}\odot Y(a_j)
\end{equation} and
\begin{equation}
\label{eq system coYoneda}
\widehat{Y}(a_k)=\oplus_{a_k\leq a_l}d_{k,l}\odot \widehat{Y}(a_l)
\end{equation}
\end{proposition}

\begin{proof}
The fact that $d^2_{\min}=d_{\min}$ is equivalent to 
\begin{equation}
    d_{i,k}=\min_j\{d_{i,j}+d_{j,k}\}.
\end{equation}

We have $Y(a_k):=d(-,a_k)$ and $Y(a_j):=d(-,a_j)$ therefore eq. 64 implies 
$$Y(a_k)_i=\oplus_j d_{j,k}\odot Y(a_j)_i$$ which means

$$Y(a_k)=\oplus_j d_{j,k} \odot Y(a_j)$$. Since 
$d_{j,k}=\infty$ unless $a_j\leq a_k$ we have 
$$Y(a_k)=\oplus_{a_j\leq a_k}d_{j,k}\odot Y(a_j)$$.

Analogously for $d^t$ we have 
$d^t_{i,k}=\min_l\{d^t_{i,l}+d^t_{l,k}\} \iff
d_{k,i}=\min_l\{d_{k,l}+d_{l.i}\}$. 
Recall that $\widehat{Y}(a_k):=d(a_k,-)$ and $\widehat{Y}(a_l):=d(a_l,-)$.
This implies
$$\widehat{Y}(a_k)=\oplus_l d_{k,l}\odot \widehat{Y}(a_l).$$  Since 
$d_{k,l}=\infty$ unless $a_k\leq a_l$ we have 
\[\widehat{Y}(a_k)=\oplus_{a_k\leq a_l}d_{k,l}\odot \widehat{Y}(a_l)\enspace .\qedhere\]
\end{proof}

Finally we have the following
\begin{proposition}
\label{Prop inner product}
The Funk metric $D(x,y):=\max_i\{y_i-x_i|x_i\neq \infty\}$ has the property that
$D(-,w)$ is tropically antilinear, namely 
\begin{equation}
D(\lambda_1 \odot x\oplus_{\min} \lambda_2 \odot y,z)=-\lambda_1\odot D(x,z) \oplus_{\max} -\lambda_2 \odot D(y,z)
\end{equation}

while $D(w,-)$ is linear, namely 
\begin{equation}
D(x,\lambda_1 \odot y\oplus_{\max} \lambda_2 \odot z)=\lambda_1\odot D(x,z) \oplus_{\max} \lambda_2 \odot D(y,z).
\end{equation}
\end{proposition}

\begin{proof}

We have $D(\lambda \odot x,y)=\max_i\{y_i-\lambda-x_i\}=D(x,y)-\lambda$.

We calculate $D(x\oplus_{\min}y,z)=\max_i\{z_i-\min\{x_i, y_i\}\}=\max_i\{z_i+\max_i\{-x_i,-y_i\}\}=$
$=\max\{\max_i\{z_i-x_i\},\max\{z_i-y_i\}\}=D(x,z)\oplus_{\max}D(y,z)$.

Moreover
$D(x,\lambda\odot y)=\max_i\{\lambda+y_i-x_i\}=\lambda-D(x,y)$.

Finally, $D(x,y\oplus_{\max}z)=D(x,\max\{y,z\})=\max_i\{\max\{y_i,z_i\}-x_i\}=$

$=\max_i\{\max\{y_i-x_i\},\max\{z_i-x_i\}\}=\max\{\max_i\{y_i-x_i\},\max_i\{z_i-x_i\}\}=$
$=D(x,y)\oplus_{\max} D(x,z)$
\end{proof}
This means that we can think of $D$ as a tropical inner product.

\begin{remark}
    All the results in this section hold for a general directed metric space.
\end{remark}
\subsection{$P(\mL)$ and $\widehat{P}(\mL)$ as Semantic spaces}
\label{Section Semantic spaces}

We already mentioned in the overview that we consider $Y(a_k):=d(-,a_k)$
as well as $\widehat{Y}(a_k):=d(a_k,-)$
 as encoding the meanings of text $a_k$ in accordance with the statistical semantics principal namely that texts that appear in similar contexts have similar meaning.
 The function $d(a_k,-)$ is supported on extensions of $a_k$ while   $d(-,a_k)$ is supported on restrictions of $a_k$.

However it is also the position of these vectors in $P(\mL)$ and $\widehat{P}(\mL)$ that contains semantic information since for example, $Y(a_k)=d(-,a_k)$ for $a_k$ a word is supported only on that word while $D(Y(a_k),-)$ is supported on all extensions of $a_k$.

Therefore  more generally,  we think of  $(P(\mL),D)$ and $(\widehat{P}(\mL),D^t)$ as ``semantic spaces'' giving mathematical substance to the statistical semantics hypothesis.
This point if view was advocated in \cite{BTV2021}.

We further explain our view about the syntax to semantics problem in Appendix B and show that it is located in the realm of a  deep and general duality in mathematics which in some cases appears as a duality between algebra and geometry.

It is interesting that even though the whole space $P(\mL)$ (or $Q(\mL)$) and $\widehat{P}(\mL)$ (or $ \widehat{Q}(\mL)$) appear as a spaces of meanings, texts appear only as special extremal rays.
They are the ``observable'' variables while other points of $P(\mL)$ are like ``hidden'' variables.

The systems of equations \Cref{prop linear system} \Cref{eq system Yoneda}, \Cref{eq system coYoneda} express the $(\min,+)$ linear relations satisfied by the Yoneda and co-Yoneda embedding text vectors. 
They are reminiscent of vector equations between word vectors as appeared first in \cite{mikolov2013distributed}.

Another way to think about them is as equations that implement the constraints imposed by the probabilities of extension. It is common to consider constraints giving rise to equations defining a geometric object and here we have something analogous but in $(\min,+)$ algebra. 

Moreover we note that any $(\min,+)$ linear combination can be transformed into a Boltzmann weighted usual linear combination using a small temperature parameter and the identity \Cref{eq temperature id}
$$\lim_{T \to 0} -T \log(e^{-y/T} +e^{-z/T})=\min\{y,z\}.$$

Using this  we will also show  in \Cref{co text vector}, \Cref{eq text vector} that $e^{-Y(a_k)}$ can be approximated by a Boltzmann weighted linear combinations of word vectors for the words that make up that text. We note the similarity of this with the expression of a value vector for a text in terms of  word vectors, in the attention layer of a transformer.

Notice also that from the formulation of probabilistic language models, 
vectors arise naturally, first in the  $(\min,+)$ context but later in 
Boltzmann weighted usual linear combinations (section 5.1). 

Moreover we will show in \Cref{section 6} that there is a duality  relating  $P(\mL)$ and $\widehat{P}(\mL)$ as well as $Q(\mL)$ and $\widehat{Q}(\mL)$ (they are isometric and tropically anti-isomorphic). This shows that given a corpus, the semantic information given by extensions  of texts is equivalent to that given by restrictions. 
 
We note that if the transformer is computing an approximation to $\widehat{P}(\mL)$ then the fact that  it is a convex space could explain why the gradient descent during training converges nicely.

Since the transformer computes probabilities for all possible next words to a text it is natural to think that the corresponding probabilistic language model $(\mL,\leq, \Pr)$ contains the whole free monoid generated by words and all texts appear as extremal rays of $P(\mL)$ corresponding to principal upper sets. Wrong texts are very far away from correct texts as they are very unlikely.

In that case the neural network should then learn an effective representation of $\widehat{P}(\mL)$,  which a priori has a huge dimension. How the neural network is able to 
construct an effective approximation of such a huge dimensional space is not clear to us.

From another point of view we see that if we consider that the transformer neural network is learning $\widehat{P}(\mL)$ then  we can think of training the transformer as finding a solution to the huge $(\min,+)$ system of  \Cref{eq system coYoneda}, \Cref{prop linear system}, given the coefficients $d_{j,k}$.

We will see a small example of the polyhedra   $Q(\mL)$ and $\widehat{Q}(\mL)$ as well as the dualities, in section~\ref{section 6}.

Further evidence for these spaces as semantic spaces is provided by the fact that 
they have a Heyting algebra structure (which is a generalization of a Boolean algebra) as explained in \cite{BTV2021}.

As already mentioned,  experiments using (a slight variant of ) the co-Yoneda vectors $d(-,a_k)$ were performed in \cite{liu2023meaning} where an actual transformer neural network was used to sample continuations of texts and construct the co Yoneda vectors. The authors tested these vectors on several semantic tasks and obtained very good results. 

\subsection{From one word text extensions to longer extensions}

We now explain how to go from one word extension probabilities  to the metric $d$.

 However $d$ constructed in this way does not satisfy the main assumption of probabilistic language model $\Pr(a_k|a_i)=\Pr(a_j|a_i)\Pr(a_k|a_j)$. It does satisfy that $d$ is a directed metric and therefore $d$ is a $(\min,+)$ projector.

Indeed, let $C$ be the matrix of one word extensions. Namely for texts $a_i$ and $a_j$ we put 

\begin{equation}
C(a_i,a_j) = \left\{
\begin{array}{ll}
  -\logPr(a_j|a_i) & \text{if } a_i \leq a_j \text{ and $a_j$ extends $a_i$ by a single word},\\
  
  \infty & \text{if }  a_i \leq a_j \text{ and $a_j$ extends $a_i$ by more than a single word},\\

  \infty & \text{if $a_i$ and $a_j$ are not comparable} .
\end{array}
\right.
\end{equation}

Let $\Id$ denote the matrix with $\Id_{i,i}=0$ and $\Id_{i,j}=\infty$ for $i\neq j$. $\Id$ is the identity matrix in the $(\min,+)$ matrix semiring. Indeed $C\Id=\Id C=C$.  We note that $C_{i,i}=0$ 
and therefore $C \oplus \Id =C$.

In that case the  tropical power $C^l$ computes  distances for up to $l$ word extensions. 

If we bound the number of words in the extension to say $k$ then $C^k=C^{k+1}$ and $d=C^k$ is our metric. 
\begin{proposition}
Let $C$ be such that $C_{i,i}=0$.
If $d=C^k=C^{k+1}$ then 
\begin{equation}
x=dx \iff x=Cx
\end{equation}
\end{proposition}

\begin{proof}

If $x=dx=C^kx$ then $Cx=C^{k+1}x=C^k x=dx=x$ therefore solutions of $x=dx$ are also solutions of $x=Cx$. 

On the other hand if $x=Cx$ then $C^kx=x$ i.e. $dx=x$.
\end{proof}

Since the diagonal entries of $d$, and $C$, are equal to $0$, the equations
$x=dx$ and $x=Cx$ are equivalent to $x\geq dx$ and $x\geq Cx$, respectively.
These two systems of inequalities describe the same polyhedron.

\section{Compatibility of $P(\mL)$ with adding more texts}

When training the neural network to learn by predicting continuations of texts we add more and more text. Moreover we have already mentioned in \Cref{re grading} that it is natural to grade $\mL$ by word length of texts.
It is therefore important to understand how 
$P(\mL)$ changes as we add more and more text. 

We have the following:

\begin{proposition}
If a probabilistic language model $(\mL_1,d_1)$ is extended to $(\mL_2,d_2)$, namely if there is an isometric embedding $\phi:(\mL_1,d_1) \hookrightarrow (\mL_2,d_2)$ then there is 
an isometric embedding $\widetilde{\phi}:(P(\mL_1)),D_1) \hookrightarrow (P(\mL_2),D_2)$ such that
    $\widetilde{\phi}(Y_1(a))=Y_2(\phi(a))$. 
    Moreover $\widetilde{\phi}(P(\mL_1))$ is a retraction (i.e. a non-expansive $(\min,+)$ projection) of $P(\mL_2)$ 
\end{proposition}

We will prove this in full generality and derive the probabilistic language model case  as a special case. 

\begin{theorem}\label{th-3}
    Let $\phi:(X_1,\delta_1) \hookrightarrow (X_2,\delta_2)$ be an isometric embedding of discrete, finite, directed metric spaces, then there is an isometric embedding $\widetilde{\phi}:(P(X_1),\Delta_1) \hookrightarrow (P(X_2),\Delta_2)$ compatible with the Yoneda isometric embeddings $Y_1:X_1 \to P(X_1)$ and $Y_2:X_2 \to P(X_2)$, namely 
    $\widetilde{\phi}(Y_1(a))=Y_2(\phi(a))$. 
    Moreover $\widetilde{\phi}(P(X_1))$ is a retraction (i.e. a non-expansive $(\min,+)$ projection) of $P(X_2)$.

\end{theorem}

\begin{proof}
Say $X_1:=\{a_1\dots a_n\}$ and $X_2:=\{b_1\dots b_n,b_{n+1},\dots b_{n+k}\}$, where $b_j=\phi(a_j)$ for $j=1 \dots n$. Recall that we have $P(X_1)=\Im(\delta_1)$ is the span of $Y_1(a_j):=\delta_1(-,a_j)$ and 
$P(X_2)=\Im(\delta_2)$ is the span of $Y_2(b_j):=\delta_2(-,b_j)$.

Let $e_m:=(\infty,\dots,0,\dots,\infty)$ for $m=1,\dots n$, so that
$e_1,\dots,e_n$ is a basis (free and generating family)
of the module $(\Rmin)^n$ of the (min,+) semifield $\Rmin$.
We define 
\[\widetilde{\phi}(\oplus_{m=1}^n x_m\odot e_m):=\oplus_{m=1}^n x_m\odot \delta_2(-,b_m)\]

We now show that 
 \[\widetilde{\phi}(\delta_1(-,a_i))=\delta_2(-,b_i).\]
for $i=1,\dots,n$.
Indeed, 
\[\widetilde{\phi}(\delta_1(-,a_i))=\tilde{\phi}(\oplus_{j=1}^n \delta_1(a_j,a_i)\odot e_j)=
\oplus_{j=1}^n \delta_1(a_j,a_i)\delta_2(-,b_j)=\delta_2(-,b_i).\]
Indeed, the last equality holds since 
\[\oplus _{j=1}^n \delta_1(a_j,a_i)\odot \delta_2(b_l,b_j)=
\oplus _{j=1}^n \delta_2(b_j,b_i)\odot \delta_2(b_l,b_j)=\delta_2(b_l,b_i),\,\]
in which the last equality follows from the fact that $\delta_2$ is a (min,+)
idempotent.

Note that $\widetilde{\phi}$ is well defined since any $x\in (\Rmin)^n$
has a unique expression in the basis $e_k, k=1,\dots , n$. If we attempted to define it directly on the $(\min,+)$ module spanned by the vectors $\delta_1(-,a_i)$ we would have to deal with the complication that $x\in \mathbb{R}^n$ does not always have a unique expression as a $(\min,+)$ combination of these vectors. In fact, one can show that only vectors in the interior of $P(\mL_1)$ would have such unique expressions.

We now check that $\tilde{\phi}$ is an isometric embedding.
We want to check that 
\begin{equation}
    \Delta_2(\widetilde{\phi}(x),\widetilde{\phi}(y))=\Delta_1(x,y)
\end{equation}    
Recall that     $\Delta_1(x,y)=\max_{j=1}^n\{y_j-x_j|x_j\neq \infty\}$. Moreover
$\Delta_2(\widetilde{\phi}(x),\widetilde{\phi}(y))=\max_{j=1}^{n+k}\{\widetilde{y}_j-\widetilde{x}_j|x_j\neq \infty\}$.

From the definition of the $\widetilde{x}_j$ the result follows.

Finally we define the retraction $\mR:P(X_2) \to P(X_2)$ by
\begin{equation}
\mR:=\bigoplus_{j=1}^n
    \Delta_2(-,Y(b_j))\odot \Delta_2(Y(b_j),-)
    \end{equation}
Note that as a matrix 
\begin{equation}
    \mR_{i,k}=R(Y(b_i),Y(b_k)=\bigoplus_{j=1}^n
    \delta_2(b_i,b_j)\odot \delta_2(b_j,b_k)\enspace .
\end{equation}
We need to check that $\mR^2=\mR$, $\Im(\mR)=\widetilde{\phi}(P(X_1))$ and $\mR$ is non-expansive.
Let us check first that $\mR^2=\mR$:

$$\mR^2(Y_2(b_k),Y_2(b_l))=\oplus_{m=1}^n \mR(Y_2(b_k),Y_2(b_m))\odot \mR(Y_2(b_m),Y_2(b_l))=$$
$$\oplus_{m,j_1,j_2=1}^n\delta_2(b_k.b_{j_1})+
\delta_2(b_{j_1}.b_m)+\delta_2(b_m.b_{j_2})+\delta_2(b_m,b_{j_2})+\delta_2(b_{j_2},b_l)=$$
$$\delta_2(b_k,b_l)=\mR(b_k,b_l).$$
Where we have used the fact that 
$$\oplus_{l=1}^n \delta_2(b_k,b_l)+\delta_2(b_l,b_m)=\oplus_{l=1}^n \delta_1(a_k,a_l)+\delta_1(a_l,a_m)=
\delta_1(a_k,a_m)=\delta_2(b_k,b_m).$$
 Next notice that clearly
 $\Im(\mR)\subset \Span_{j=1}^n\{\delta_2(-,b_j)\}=\widetilde{\phi}(P(\mL_1)$.
Moreover we claim that 
\[\mR(\delta_2(-,b_k))=\delta_2(-,b_k)\]
 
Indeed \[\mR(\delta_1(-,b_k))=\oplus_{j=1}^n\delta_2(b_j,b_k)\odot\delta_2(-,b_j).\]
and thus
\[\mR(\delta_1(b_l,b_k))=\oplus_{j=1}^n\delta_2(b_j,b_k)+\delta_2(b_l,b_j)=\delta_2(b_l,b_k),\]
proving the claim.

 Therefore  $\Span_{j=1}^n\{\delta_2(-,b_j)\}\subset \widetilde{\phi}(P(\mL_1))\subset \Im(\mR)$ showing that 
 $\Im(\mR)=\widetilde{\phi}(P(\mL_1)$. 

Finally we check that $\mR$ is non-expansive, namely that
\[\Delta_2(\mR(x),\mR(y))\leq \Delta_2(x,y)\]
To that end note that $\mR$ is order preserving and also $\mR(\alpha\odot x)=\alpha\odot \mR(x)$.
Indeed both of these statements follow from the $(\min,+)$ linearity of $\mR$.

In particular  $x\leq y \iff x\oplus y=x$ which implies that $\mR(x\oplus y)=\mR(x)$ and therefore $\mR(x)\oplus \mR(y)=\mR(x)$ which 
means $\mR(x) \leq \mR(y)$.

 Recall now that $\Delta_2(x,y)=\inf \{\lambda:x\leq \lambda \odot y\}=\max_i\{y_i-x_i|x_i \neq \infty\}$.  
Then \[x\leq \Delta_2(x,y)\odot y \implies \mR(x) \leq \mR(\Delta_2(x,y)\odot y) \implies \mR(x)\leq \Delta_2(x,y)\odot \mR(y)\] therefore 
$\Delta_2(\mR(x),\mR(y)\leq \Delta_2(x,y)$

\end{proof}
\begin{remark}
Since $\tilde{\phi}(Y_1(a_j)):=Y_2(\phi(a_j)$ we have, if $x:=\oplus_{j=1}^n x_j\odot Y_1(a_j)$, 

$$\widetilde{\phi}(x):=\oplus_{j=1}^n x_j \odot Y_2(\phi(a_j))=\oplus_{j=1}^n x_j \odot Y_2(b_j)=
\oplus_{j=1}^{n+k} \tilde{x}_j \odot Y_2(b_j),$$
where $\widetilde{x}_j:=x_j$ for 
$j=1, \dots, n$ and $\widetilde{x}_j=\infty$ for $j=n+1,\dots, {n+k}$.
So in these coordinates $P(X_1)$ is cut out inside $P(X_2)$ by the equations $\tilde{x}_j=\infty$ for $j=n+1,\dots, {n+k}$. In this sense, it constitutes
a ``face'' of $P(X_2)$ of (projective) dimension $|X_1|-1$.
\end{remark}

\begin{remark}
Note that if $x\in P(\mL_2)$ where 
$x:\mL_2 \to (-\infty,\infty]$ and $x_i:=x(b_i)$ then
\begin{equation}
    \mR(x)=\bigoplus_{j=1}^n\Delta_2(Y_2(b_j),x)\odot \Delta_2(-,Y(b_j)) :\mL_2 \to (-\infty,\infty]
\end{equation}
and for $i=1\dots, n+l$
\begin{equation}
\mR(x)_i:=\mR(x)(b_i)=\bigoplus_{j=1}^n \Delta_2(Y_2(b_i),Y_2(b_j))\odot \Delta_2(Y_2(b_j),x)=\bigoplus_{j=1}^n d_2(b_i,b_j)\odot x_j
\end{equation}
Therefore
\begin{align}
\mR(x)&=\bigoplus_{i=1}^{n+l}  \mR(x)_i \odot Y_2(b_i)=
\bigoplus_{i=1}^{n+l} \bigoplus_{j=1}^n d_2(b_i,b_j)\odot x_j \odot Y_2(b_i)\nonumber \\
& =\bigoplus_{i=1}^{n+l} \bigoplus_{j=1}^n d_2(b_i,b_j)\odot \Delta_2(Y(b_j),x) \odot Y_2(b_i)
\end{align}

\end{remark}

\subsection{Approximation of a text vector in terms of word vectors}

Let us see how \Cref{th-3} applies to the probabilistic language model case. 

\begin{corollary}

Let $\mL_1:=\{a_1\dots a_n\}$ and $\mL_2:=\{b_1\dots b_n,b_{n+1},\dots b_{n+l}\}$, be probabilistic language models and $\phi:\mL_1 \to \mL_2$ an isometric embedding  where $b_j=\phi(a_j)$ for $j=1 \dots n$. 
Let $Y_1: (\mL_1,d_1) \to (P(\mL_1),D_1)$ and 
$Y_2: (\mL_2,d_2) \to (P(\mL_2),D_2)$ be the Yoneda isometric embeddings.

Let $\mR:P(\mL_2)\to P(\mL_2)$ be the non-expansive projection of \Cref{th-3} given by 
\begin{equation}
\mR:=\bigoplus_{j=1}^n D_2(-,Y_2(b_j))\odot D_2(Y_2(b_j),-).
\end{equation}
Then for $i,k=1,\dots, n+l$
\begin{equation}
\mR(Y_2(b_k))_i=\bigoplus_{j=1}^n
    d_2(b_i,b_j)\odot d_2(b_j,b_k)
\end{equation}

and 
\begin{equation}
    \mR(Y_2(b_k))=\bigoplus_{i=1}^{n+l} \mR(Y_2(b_k))_i \odot Y_2(b_i)=\bigoplus_{i=1}^{n+l} \bigoplus_{j=1}^n d_2(b_i,b_j)\odot d_2(b_j,b_k) \odot Y_2(b_i).
\end{equation}
or equivalently
\begin{equation}
    \mR(Y_2(b_k))=\bigoplus_{b_i\leq b_j \leq b_k} d_2(b_i,b_j)\odot d_2(b_j,b_k) \odot Y_2(b_i).
\end{equation}
\end{corollary}
\begin{proof}
We have 
$$\mR=\bigoplus_{j=1}^n
    D_2(-,Y_2(b_j))\odot D_2(Y_2(b_j),-),$$ 

    Applying to $Y_2(b_k)$ for $k=1,\dots n+l$ we get 
    $$\mR(Y_2(b_k))=\bigoplus_{j=1}^n
    D_2(-,Y_2(b_j))\odot D_2(Y_2(b_j),Y_2(b_k)).$$
    Since $Y_2$ is an isometric embedding we have
$$\mR(Y_2(b_k))=\bigoplus_{j=1}^n
    D_2(-,Y_2(b_j))\odot d_2(b_j,b_k): \mL_2 \to (-\infty,\infty].$$

    Therefore for $i=1,\dots, n+l$
    \begin{align}
    \mR(Y_2(b_k))_i&=\mR(Y_2(b_k))(b_i)=\bigoplus_{j=1}^n
    D_2(Y_2(b_i),Y(b_j))\odot d_2(b_j,b_k)\nonumber\\
    &=
    \bigoplus_{j=1}^n
    d_2(b_i,b_j)\odot d_2(b_j,b_k)
\end{align}
 Consequently 
 \begin{equation}
     \mR(Y_2(b_k))=\oplus_{i=1}^{n+l}\mR(Y_2(b_k))_j \odot Y(b_i)=\bigoplus_{i=1}^{n+l} \bigoplus_{j=1}^n d_2(b_i,b_j)\odot d_2(b_j,b_k) \odot Y_2(b_i).
 \end{equation}

 or equivalently
\begin{equation}
    \mR(Y_2(b_k))=\bigoplus_{b_i\leq b_j \leq b_k} d_2(b_i,b_j)\odot d_2(b_j,b_k) \odot Y_2(b_i). \qedhere
\end{equation}
\end{proof}

\begin{remark}

    We see from eq 78,
$\mR(Y_2(b_k))_i=\bigoplus_{j=1}^n
    d_2(b_i,b_j)\odot d_2(b_j,b_k)$,
    that only summands such that $b_i\leq b_j \leq  b_k$, will be finite.
\end{remark}

\begin{remark}
Recall that, according to \Cref{th-3}, $\mR$ is a non-expansive map therefore 
\begin{equation}
D(\mR(Y(b_k)),\mR(Y(b_l)))\leq D(Y(b_k),Y(b_l)).
\end{equation}
    
\end{remark}

We can use the previous proposition in order to approximate a text vector by the vectors corresponding to words making up that text. 

\begin{corollary}
Let $\mL:=\{b_1,\dots, b_N\}$ be a probabilistic language model and let  $W:=\{w_1\dots,w_m\}$ be the set of words identified with $b_1,\dots, b_m$ and considered as a probabilistic language model with all pairwise distances equal to infinity. 
Let $Y:\mL \to P(\mL)$ be the Yoneda embedding.  
Let $\mR:P(\mL)\to P(\mL)$ be the non-expansive projection given by 
\begin{equation}
\mR:=\bigoplus_{j=1}^m D(-,Y(w_j))\odot D(Y(w_j),-).
\end{equation}

Consider $Y(b_k) \in P(\mL)$, then 
for $i,k=1,\dots, N$
\begin{equation}
\mR(Y_2(b_k))_i=d_2(w_i,b_k)
\end{equation}

and

\begin{equation}
\mR(Y_2(b_k))=\bigoplus_{i=1}^N d_2(w_i,b_k)\odot Y_2(w_i)=\bigoplus_{w_i\leq b_k} d_2(w_i,b_k)\odot Y_2(w_i)
\
\end{equation}

\end{corollary}

\begin{proof}
Consider the projection 
$\mR:P(\mL) \to P(\mL)$ given 
$$\mR=\bigoplus_{j=1}^m
    D(-,Y(w_j))\odot D(Y(w_j),-),$$ 

    where $\Im(R)=\widetilde{\phi}(P(W))$.

We have identified $b_j$ with $w_j$ for $j=1,\dots,m$ therefore from corollary 5 we have
for $i, k=1,\dots,N$ 

\begin{equation}
\mR(Y_2(b_k))_i=\bigoplus_{j=1}^l
    d_2(b_i,w_j)\odot d_2(w_j,b_k).
\end{equation}

However $d_2(b_i,w_j)$ is finite only of j=i and $w_j=b_i$. In that case
$d_2(b_i,w_i)=0$.
Therefore for $i=1,\dots, N$

\begin{equation}
\mR(Y_2(b_k))_i=d(w_i,b_k).
\end{equation}

Consequently 
\begin{equation}
\label{eq text to word}
    \mR(Y_2(b_k))=\bigoplus_{i=1}^N d_2(w_i,b_k)\odot Y_2(w_i)=\bigoplus_{w_i\leq b_k} d_2(w_i,b_k)\odot Y_2(w_i)\enspace .\qedhere
\end{equation}
\end{proof}

\begin{corollary}
\label{co text vector}
Let $\mL:=\{b_1,\dots, b_N\}$ be a probabilistic language model and let  $W:=\{w_1\dots,w_m\}$ be the set of words identified with $b_1,\dots, b_m$ and considered as a probabilistic language model with all pairwise distances equal to infinity. 
Let $Y:\mL \to P(\mL)$ be the Yoneda embedding.  
 Let $T\geq 0$ be a parameter (which is usually called temperature), then 
 we have 

\begin{equation}
\mR(Y(b_k))=\lim_{T\to 0}-T \log(\sum_{w_i\leq b_k}
e^{-\frac{d(w_i,b_k)}{T}} e^{-\frac{Y(w_i)}{T}})
\end{equation}
Therefore for small $T$ we have
\begin{equation}
\label{eq text vector}
e^{-\frac{\mR(Y(b_k))}{T}}\approx \sum_{w_i\leq b_k}
e^{-\frac{d(w_i,b_k)}{T}} e^{-\frac{Y(w_i)}{T}}
\end{equation}

\end{corollary}
\begin{proof}
    Recall the identity 
\begin{equation}
\label{eq temperature id}
\lim_{T \to 0} -T \log(e^{-y/T} +e^{-z/T})=\min\{y,z\}.
\end{equation}

Then \cref{eq text to word} implies the result.
\end{proof}
\begin{remark}
    \Cref{eq text vector}  is similar to the expression for a text value vector in terms of word value vectors as computed in the attention module of a transformer.
\end{remark}
\begin{remark}
As already mentioned, it is natural to filter the probabilistic language $\mL$ by the word length of texts. Define $\mL_k$ to be the set of texts on $\mL$ that have word length up to $k$. $\mL_1$ will be the set of words. Each $\mL_k$ inherits the structure of a probabilistic language model from $\mL$. The inclusions define isometric embeddings
$\phi_k:\mL_k \to \mL_{k+1}$.
Then 
we can consider the non-expansive projections $\mR_k:P(\mL_{k+1}) \to P(\mL_{k+1})$ where $\Im(\mR_{k+1})=\widetilde{\phi}_k(P(\mL_k))$.
\end{remark}

\section{Duality between text extensions and restrictions}
\label{section 6}
We have already considered the $(\min,+)$ semifield $\R_{\min}:=((-\infty,+\infty],\oplus_{\min},\odot)$.
To express duality results though, it will be convenient to work with the completed $(\min,+)$ semiring
$\bar{\R}_{\min}:=([-\infty,+\infty],\oplus_{\min},\odot)$ where as before
$s\oplus_{\min} t:=\min\{s,t\}$ and $s\odot t:=s+t$
but we need to further determine how $-\infty$ and $\infty$ interact. 

Indeed we specify that the element $+\infty$ remains absorbing, so $+\infty + s=+\infty$ holds for all
element $s$, and in particular $(+\infty)+(-\infty) = +\infty$.
The definition of $d_{\min}$ in~\eqref{e-def-dmin}
extends to this semiring.
We also need to extend definitions of $P(\mL)$ and $\widehat{P}(\mL)$:

\begin{definition}
Let $P^-(\mL),D)$ be the directed metric polyhedron
\begin{equation}
P^-(\mL):=\{x=(x_1,\dots,x_n) \in \{\R\cup \{\infty,-\infty\}\}^n \text{\textbackslash}\{(\infty,\dots,\infty)\}| x_i \leq x_j + d_{i,j}\} .
\end{equation}
Moreover  let $\widehat{P}^-(\mL),D^t)$ 
be the directed metric polyhedron
\begin{equation}
\widehat{P}^-(\mL):=\{y=(y_1,\dots,y_n) \in \{\R\cup \\\{\infty,-\infty \}^n \text{\textbackslash}\{(\infty,\dots,\infty)\}| y_i \leq y_j + d_{j,i}\}.
\end{equation}
\end{definition}
\begin{remark}
\label{rem absorbing}
Recall  that we added the case of a directed metric which can also take the value $-\infty$ in Remark 1 and now $D$ is such a metric.
Moreover we specified that $+\infty$ is absorbing in $\R_{\min}$. However the Funk metric $D$ is defined using $\max$ so we have to specify further our convention to cover expressions that contain both $\min$ and $\max$. 
For that, we simply use the relation $\max(s,t)=-\min\{-s,-t\}$ to transform any $\max$ in the expression to $\min$ so that we end up with an expression containing only $\min$. Then we compute using the $(\min,+)$ convention that $+\infty$ is absorbing.

Equivalently we can use the same relation to transform any expression to one that contains only $\max$. Then  using $-\infty$ as the absorbing element gives the same answer. 

\end{remark}
Now, analogously to \Cref{prop-11}, if we consider $d_{\min}$  and $d^t_{\min}$ acting on $\{\R\cup \{\infty,-\infty\}\}^n \text{\textbackslash}\{(\infty,\dots,\infty)\}$ then we have 
\begin{proposition}
The polyhedron $P^-(\mL)$ is equal to $\Im(d_{min})=\Fix(d_{\min})$ and the polyhedron $\widehat{P}^-(\mL)$ is equal to 
$\Im(d^t_{min})=\Fix(d^t_{\min})$.
\end{proposition}

\begin{definition}
  Define the pair of maps $(A,B)$ as follows. If $y: \mL \to [-\infty,\infty]$   and  
$x: \mL \to [-\infty,\infty]$  then
\begin{equation}
A(y):=d_{\min}(-y),
B(x):=d^t _{\min}(-x)
\end{equation}

Or in coordinates
\begin{equation}
A(y)_i:=\min_j\{d_{i,j}-y_j\},
B(x)_j:=\min_i\{d_{i,j}-x_i\}
\end{equation}
\end{definition}

We also denote by $D^t$ the transpose metric with $D^t(x,y):=D(y,x)$.

In fact we will see that  $A$  and $B$ on non-expansive maps with respect to these metrics.

The pair $(A,B)$ forms an adjunction in the categorical or metric sense:

\begin{proposition}
\label{prop adjunction AB}
If  $x: \mL \to [-\infty,\infty]$ and  $y: \mL \to [-\infty,\infty]$  then 
we have 
\begin{equation}
D(A y,x)=D^t(y,B x)
\end{equation}

\end{proposition}

\begin{proof}
$D(Ay,x)=\max_i\{x_i-\min_j\{d_{i,j}-y_j\}\}=-\min_i\{\min_j\{d_{i,j}-y_j\}-x_i\}=
-\min_j\{\min_i\{d_{i,j}-x_i\}-y_j\}=
\max\{y_j-\min_i\{d_{i,j}-x_i\}\}=
D(B x,y)=D^t(y,B x)$.
\end{proof}

\begin{remark}
\begin{enumerate}
\item
 Note the resemblance of the pair of adjoint maps $(A,B)$ with the Legendre-Fenchel transform where the metric is replaced by the inner product
 of a vector space. 
\item
We note, for purposes of developing intuition, that the pair of adjoint maps $(A,B)$ is similar to a pair of adjoint linear maps $(A,A^*)$ on a vector space with inner product 
 $\<-,-\>$. Indeed in the usual linear algebra case 
$\<Au,v\>=\<u,A^*v\>$. Moreover   we have already seen in \Cref{prop-15}
that $D$ is a kind of tropical inner product. There is a crucial difference though that  $\<v,v\>=|v|^2$ while $D(x,x)=0$. This reflects the fact that to go from usual algebra to tropical algebra we apply $-\log$.
\end{enumerate}
\end{remark}

We now have the following
\begin{proposition}
We have $ABA=A$ and $BAB=B$ which implies that $AB$ and $BA$ are idempotent. 
\end{proposition}

\begin{proof}

This follows from the fact that $D(A y,x)=D^t(y,B x)$.

Indeed 
$D(ABA y,A y)=D^t(BA y,BA y)=0$ 

and 
$D(A y,ABAy)=D^t(BAy,BA y)=0$. 
Therefore $ABA x=A x$. The equality $BAB=B$ is shown analogously.
\end{proof}

Let us now compute the fixed parts of the adjunction $\Fix(AB)$ and $\Fix(BA)$.

\begin{proposition}
We have $\Fix(AB)=\Im(A)=\Im(d_{\min})$ and $\Fix(BA)=\Im(B)=\Im(d_{\min}^t)$.
\end{proposition}
\begin{proof}
 This follows from the fact that $ABA=A$. Indeed clearly $\Im(A)\subset \Fix(AB)$. 
Moreover $\Fix(AB)\subset \Im(A)$ since $AB(x)=x$ says that $x\in \Im(A)$. Analogously for $BA$.
\end{proof}

In this case, due to the fact that $d_{\min}$ is an idempotent we can more explicitly compute the maps $A$ and $B$
\begin{proposition}\label{prop-antiisom}
We have that 
\begin{equation}
A:\Im(d_{\min}^t)=\widehat{P}^-(\mL) \to \Im(d_{\min})=P^-(\mL) \text { is given by } A(y)=-y
\end{equation}
and 
\begin{equation}
B:\Im(d_{\min})=P^-(\mL) \to \Im(d_{\min}^t)=\widehat{P}^-(\mL) \text{ is given by }
B(x)=-x.
\end{equation}
\end{proposition}
\begin{proof}
Consider  $x\in \Im(d_{\min})=Fix(d_{\min})$. We have 
$d_{\min}(x)=x \iff 
x_i=\min_j\{d_{i,j}+x_j\} \iff 
-x_j=\min_j\{d_{i,j}-x_i\} \iff
d_{\min}^t(-x)=-x$. 
Therefore
$B(y)=d_{\min}^t(-x)=-x$.

Analogously consider $y\in \Im(d_{\min}^t)$. We have 
$d_{\min}^t(y)=y \iff
d_{\min}(-y)=-y$. Therefore
$A(y)=d_{\min}(-y)=-y$.
 \end{proof}
\begin{remark}
Note that we can directly check the adjunction of  \Cref{prop adjunction AB} using our explicit formula from 
\Cref{prop-antiisom}. 
Indeed $D(Ax,y)=D(-x,y)=\max_i\{y_i+x_i|x_i\neq -\infty\}$.
Moreover 
$D^t(x,By)=D(-y,x)=\max_i\{x_i+y_i|y_i\neq -\infty\}$.
Since we have a $\max$ expression, $-\infty$ is absorbing (see \Cref{rem absorbing})  and consequently if $x_i=-\infty$ of if $y_i=-\infty$ then $x_i+y_i=-\infty$ therefore both these conditions can be ignored for taking the max and we get 
$D(-x,y)=D(-y,x)$.
\end{remark}
 
The following  theorem has been proved in~\cite{sturmfels} and~\cite{cgq02}
from different points of view and in different generalities. Another approach using category theory was used in Willerton \cite{willerton2013tight}.

Here we take advantage of the explicit computation in \Cref{prop-antiisom} which is true because $d_{\min}^2=d_{\min}$.
 
 \begin{theorem}\label{th-4}
We have that $$A:\Fix(BA)=\Im(B)=\Im(d_{\min}^t)=\widehat{P}^-(\mL) \to \Fix(AB)=\Im(A)=\Im(d_{\min})=P(\mL)$$ and 
$$B:\Fix(AB)=\Im(A)=\Im(d_{\min})=P(\mL) \to \Fix(BA)=\Im(B)=\Im(d_{\min}^t)=\widehat{P}^-(\mL)$$
are anti-isomorphisms. In other words they are one to one and onto and inverses.  They are isometries, namely
$D(A y,A y')=D^t(y,y')$.
Finally we have
\begin{equation}
A(\lambda \odot y)=-\lambda \odot A(y),
\end{equation}
\begin{equation}
A(y \oplus_{\min} y')=A(y)\oplus_{\max} A(y') 
\end{equation}
and
\begin{equation}
A(y \oplus_{\max} y')=A(y)\oplus_{\min} A(y')
\end{equation}
and similarly for $B$.
\end{theorem} 

\begin{proof}
From Proposition 16 
\begin{equation}
A:\Im(d_{\min}^t)=\widehat{P}(\mL) \to \Im(d_{\min})=P(\mL) \text { is given by } A(y)=-y
\end{equation}
and 
\begin{equation}
B:\Im(d_{\min})=P(\mL) \to \Im(d_{\min}^t)=\widehat{P}(\mL) \text{ is given by }
B(x)=-x.
\end{equation}
therefore $A$ and $B$ are one on one and onto and inverses.

Moreover $D(Ay,Ay')=D(-y,-y')=
\max_i\{y'_i-y_i\}=D((y',y)=D^t(y,y')$.
Furthermore,
$A(\lambda\odot y)_i=-(\lambda+y_i)=-\lambda\odot A(y)_i$

$A(y \oplus_{\max} y')_i=-\max\{y_i,y'_i\}=
\min\{-y_i,-y'_i\}=A(y)_i\oplus_{\min} A(y')_i.$

$A(y \oplus_{\min} y')_i=-\min\{y_i,y'_i\}=
\max\{-y_i,-y'_i\}=A(y)_i\oplus_{\max} A(y')_i.$

\end{proof}
(Note that $\Im(d_{\min})$ is $(\max,+)$ closed; this follows from \Cref{prop-lattice-completion}.)

We have then that the $(\min+)$ column span $P^-(\mL)$, of $d_{\min}$ is anti isomorphic to the  $(\min,+)$ row span $\widehat{P}^-(\mL)$ of $d_{\min}$   (as $\bar{\R}_{\min}$  modules) by the two inverse maps $A$ and $B$ and moreover they are isometric when considered with the directed metrics $D$ and $D^t$ respectively.
(Recall also that in \Cref{Prop inner product} we saw that $D$ can be considered as tropical inner product.)

We will see an example of this bellow. 
First though we would like to make this map more explicit with respect to the rows and columns of the matrix $d$.

\begin{proposition}
\label{prop duality formula}
Consider $x\in \Im(d_{\min})$. We have that 
\begin{equation}
x=\oplus_j x_j \odot d(-,a_j) \textit{ and  then } 
B(x)=-x=\oplus_j -x_j \odot d(a_j,-).
\label{eq-first}
    \end{equation}
    In particular if $x=d(-,a_k)$ then 
    \begin{equation}
    \label{eq duality 1}
    d(-,a_k)=\oplus_{a_j\leq a_k}d(a_j,a_k) \odot d(-,a_j)
     \end{equation}
    and
    \begin{equation}
    \label{eq duality 2}
    -d(-,a_k)=\oplus_{a_j\leq a_k}-d(a_j,a_k) \odot d(a_j,-)
    \end{equation}
    Analogously for $y\in \Im(d^t)$ we have
    \begin{equation}
    y=\oplus_i y_i \odot d(a_i,-) \textit{ and then } 
    A(y)=-y=\oplus_i -y_i \odot d(-,a_i).
    \end{equation}
    
    In particular if $y=d(a_k,-)$ then
    \begin{equation}
    d(a_k,-)=\oplus_{a_k\leq a_i}d(a_k,a_i) \odot d(a_i,-)
    \end{equation}
    and
    \begin{equation}
    -d(a_k,-)=\oplus_{a_k\leq a_i}-d(a_k,a_i) \odot d(-,a_i).
    \end{equation}

\end{proposition}
 \begin{proof}
     We have $x\in \Im(d_{\min}) \iff d_{\min}x=x \iff x=\oplus_j x_j \odot d(-,a_j)$.
From \Cref{prop-antiisom} we then have $d_{\min}^t(-x)=-x$ which is 
equivalent to $-x=\oplus_j -x_i \odot d(a_j,-)$. This proves \eqref{eq-first}.

Now if $x:=d(-,a_k)$ then $x_j=x(a_j)=d(a_j,a_k)=d_{j,k}$. Then from \Cref{prop linear system} we have 
$d(-,a_k)=\oplus_{a_j\leq a_k}d_{j,k} \odot d(-,a_j)$ therefore
from  \eqref{eq-first} it follows that 
$ -d(-,a_k)=\oplus_{a_j\leq a_k}-d(a_j,a_k) \odot d(a_j,-)$.
The proof for $y\in \Im(d_{\min}^t)$ and for $y:=d(a_k,-)$ is analogous.
 \end{proof}

\begin{remark}
    Note that all results in this section hold for a general directed metric space.
\end{remark}

\begin{example}
We now show a simple example of a probabilistic language model $(\mL,d_1)$ along with $P(\mL)$ and $\widehat{P}(\mL)$. We will also  see the correspondence of extremal rays with connected lower sets for $P(\mL)$ and connected upper sets for $\widehat{P}(\mL)$ as described in \Cref{th-2}. 

We will actually consider the corresponding polyhedral cones $Q(\mL)$ and $\widehat{Q}(\mL)$ and show in the figures  the polyhedra $Q_0(\mL)$ and $\widehat{Q}^0(\mL)$ (\Cref{def Q_0}) which are their intersections with the unit simplex. 

We will further illustrate the duality between 
completions $P^-(\mL)$ and  $\widehat{P}^-(\mL)$ by making a uniform approximation of infinities in $d$ with a big number $M$.

Indeed consider the corpus to be $\mL:=\{\text{red, colour, red colour}\}$.
Denote ``red'' by ``r'', ``colour'' by ``c'' and ``red colour'' by ``rc''.

Let the metric $d$ be given by eq.~\eqref{e-exd}:
 \begin{align}
 d = \bordermatrix{ & r  & c & rc \cr
   r& 0 & \infty& \log 2\cr
   c& \infty & 0 &\log 3 \cr
   rc& \infty & \infty & 0
 }
 \label{e-exd}
 \end{align}
 Recall that in general $e^{-d_{i,j}}=\Pr(a_j|a_i)$
and thus the corresponding matrix of probabilities of extensions is
 \begin{align}
 \Pr= \bordermatrix{ & r  & c & rc \cr
   r& 1 & 0 & \frac{1}{2}\cr
   c& 0 & 1 &\frac{1}{3} \cr
   rc& 0 & 0 & 1
 }
 \label{e-exPr}
 \end{align}
 
 This means for example that $\Pr(rc|r)=\frac{1}{2}$ and $\Pr(c|r)=0$, while $P(r|r)=1$.

Recall that the equations for $P(\mL)=\Im(d_{\min})$ are (\Cref{def P}): $x_i\leq d_{i,j}+x_j$. 

Letting $z_i:=e^{-x_i}$ we have that the equations for 
$Q(\mL)$ are (\Cref{def Q}): $z_i\geq e^{-d_{i,j}}z_j$.

Therefore in our case we get that the polyhedral cone $Q(\mL)$ is defined by  inequalities 
\begin{equation}
z_1\geq \frac{1}{2}z_3, z_2\geq \frac{1}{3}z_3, z_1\geq 0, z_2\geq 0, z_3 \geq 0.
\end{equation}
The intersection  $Q_0(\mL)$ of $Q(\mL)$ with the unit simplex is shown on the right in 
\Cref{fig-example}. Notice that it has three vertices. 

Analogously, the equations for 
$\widehat{P}(\mL)=\Im(d^t_{\min})$ (\Cref{def P}) are  $y_j\leq d_{i,j}+y_i$. 

Letting $u_i:=e^{-y_i}$ we have that the equations for 
$\widehat{Q}(\mL)$  are (\Cref{def P}) $u_j\geq e^{-d_{i,j}}u_i$.

Therefore in our case we get that the polyhedral cone $\widehat{Q}(\mL)$ is defined by  inequalities 
\begin{equation}
 u_3\geq \frac{1}{2}u_1, u_3\geq \frac{1}{3}u_2,  u_1\geq 0, u_2\geq 0, u_3 \geq 0.
\end{equation}

The intersection  $\widehat{Q}_0(\mL)$ of  $\widehat{Q}(\mL)$ with the unit simplex is shown on the left in \Cref{fig-example}.
Notice that it has four vertices.

Denote the lower set generated by ``$a$'' by $(a)_l$ and the upper set  generated by $a$ by $(a)_u$.

From \Cref{th-2}, extremal rays of $Q(\mL)$ correspond to \textit{connected lower sets} of $\mL$. There are three and they are all principal:
$(r)_l=\{r\},(c)_l=\{c\}, (rc)_l=\{r,c,rc\}$. 
These give rise to the three vertices of 
$Q_0(\mL)$ as we can see in Fig 1.
(Note that $(r,c)_l$ is not connected so it does not correspond to an extremal ray of $Q(\mL)$.

From \Cref{cor 4}, extremal rays of $\widehat{Q}(\mL)$ correspond to \textit{connected upper sets} of $\mL$. The principal ones are 
$(r)_u=\{r,rc\},(c)_u=\{c,rc\}, (rc)_u=\{rc\}$ and a non-principal one $(r,c)_u=\{r,c,rc\}$.
This extremal ray is not in the image of the Yoneda embedding. 
The corresponding four vertices of $\widehat{Q}_0(\mL)$ are shown on the left in \Cref{fig-example}.

Notice that the number of extremal rays of $P(\mL)$ and $\widehat{P}(\mL)$ are actually  different.

Now note that in general the $\R_{\min}$ module $P(\mL)=\Im(d_{\min})$ is a geometric object. In fact $Q(\mL)$ is a polyhedral cone and $Q_0(\mL)$ is a polyhedron. However the 
$\bar{\R}_{\min}$ module $P^-(\mL)$ is not obviously geometric. In order to approximate with a geometric object and be able 
to visualize the duality between the 
 $P^-(\mL)$ and $\widehat{P}^-(\mL)$
it is natural to ``truncate'' the
matrix $d$, replacing the $+\infty$ entries by a sufficiently large number
$M$, leading to the new matrix:
\begin{align}
d^M = \bordermatrix{ & r  & c & rc \cr
  r& 0 & M& \log 2\cr
  c& M & 0 &\log 3 \cr
  rc& M & M & 0
}
\label{e-exdbigM}
\end{align}
\textit{This matrix is sill a directed metric, satisfying $(d^M_{\min})^2=d^M_{\min}$ but it does not any more
represent a probabilistic language model.}

We can consider $P_M(\mL):=\Im(d_{\min}^M)$ as in  (\Cref{def P}) and $Q_M(\mL)$ as in (\Cref{def Q}).
Then the intersection $Q_{M,0}(\mL)$ of $Q_M(\mL)$ with the unit simplex is depicted in \Cref{fig-example-bigM} on the right.

Moreover we consider
$\widehat{P}_M(\mL):=\Im((d^M_{\min})^t)$ as in (\Cref{def P}) and $\widehat{Q}_M(\mL)$ as in 
(\Cref{def Q}).
Then the intersection $\widehat{Q}_{M,0}(\mL)$ of $\widehat{Q}_M(\mL)$ with the unit simplex is depicted in \Cref{fig-example-bigM} on the right.

Observe that the duality preserves the number of extreme points inside the interior of the simplex, and that the sets of~\Cref{fig-example-bigM}
converge to the sets of~\Cref{fig-example} as $M\to \infty$. 

Also note that the duality map between
$P(\mL)$ an $\widehat{P}(\mL)$ is $x_i\to y_i=-x_i$ for $i=1,2,3$. 
We also have $z_i:=e^{-x_i}$ and $u_i:=e^{-y_i}$
Therefore the map between 
the polyhedra $Q(\mL)$ and $\widehat{Q}(\mL)$ is $z_i\to u_i=\frac{1}{z_i}$ for $i=1,2,3$.

\end{example}

\colorlet{mygreen}{gray!50!black}

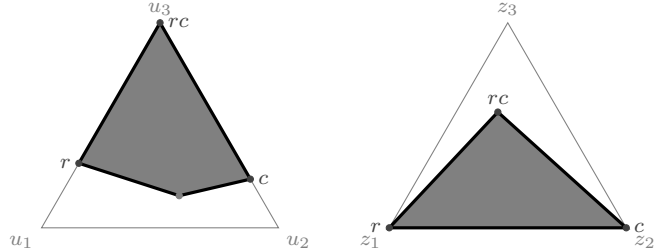
\begin{figure}[htbp]
\begin{center}\footnotesize
  \begin{minipage}[b]{0.25\textwidth}
    \renewcommand{\baryx}{u_1}
\renewcommand{\baryy}{u_2}
\renewcommand{\baryz}{u_3}

\begin{tikzpicture}%
[scale=0.45,>=triangle 45
,vtx/.style={mygreen},
ray/.style={myred}]
\equilateralnobars{7}{100};

\def\myscale{1.8}
\def\bigM{4}

\barycenter{e1}{\expo{0}}{0}{0};
\barycenter{e2}{0}{\expo{0}}{0};
\barycenter{e3}{0}{0}{\expo{0}};
\barycenter{A3}{\expo{-100}}{\expo{-100}}{\expo{0}};
\barycenter{A1}{\expo{0}}{\expo{-100}}{\expo{-2/\myscale}};
\barycenter{A2}{\expo{-100}}{\expo{0}}{\expo{-3/\myscale}};
\barycenter{PSEUDO}{\expo{2/\myscale}}{\expo{3/\myscale}}{\expo{0}};
\filldraw[gray,draw=black,opacity=0.5,very thick] (A1) -- (e3) -- (A2) -- (PSEUDO) -- cycle;

\filldraw[vtx] (A1) circle (0.75ex) node[left] {$r$}; %
\filldraw[vtx] (A2) circle (0.75ex) node[right] {$c$}; %
\filldraw[vtx] (e3) circle (0.75ex) node[above,right] {$rc$}; %
\filldraw[gray] (PSEUDO) circle (0.75ex) node[above] {};
\end{tikzpicture}
\end{minipage}\hskip 5em
  \begin{minipage}[b]{0.25\textwidth}
        \renewcommand{\baryx}{z_1}
\renewcommand{\baryy}{z_2}
\renewcommand{\baryz}{z_3}

\begin{tikzpicture}%
[scale=0.45,>=triangle 45
,vtx/.style={mygreen},
ray/.style={myred}]
\equilateralnobars{7}{100};
\def\myscale{1.8}
\barycenter{e1}{\expo{0}}{0}{0};
\barycenter{e2}{0}{\expo{0}}{0};
\barycenter{e3}{0}{0}{\expo{0}};
\barycenter{A3}{\expo{-2/\myscale}}{\expo{-3/\myscale}}{\expo{0}};
\barycenter{A1}{\expo{0}}{\expo{-100}}{\expo{-2/\myscale}};
\barycenter{A2}{\expo{-100}}{\expo{0}}{\expo{-3/\myscale}};
\barycenter{PSEUDO}{\expo{2/\myscale}}{\expo{3/\myscale}}{\expo{0}};
\filldraw[gray,draw=black,opacity=0.5,very thick] (e1) -- (e2) -- (A3) -- cycle;
\filldraw[vtx] (e1) circle (0.75ex) node[left] {$r$}; %
\filldraw[vtx] (e2) circle (0.75ex) node[right] {$c$}; %
\filldraw[vtx] (A3) circle (0.75ex) node[above] {$rc$};%
\end{tikzpicture}
\end{minipage}
\end{center}
\caption{The cross section $\widehat{Q}_0(\mL)$ of the polyhedral cone $\hat{Q}(\mL)$ arising from the metric of $d$ (left).
  Every vector $d(r,-), d(c,-),d(rc,-)$ determines an
  extreme point of the cross section, denoted by $r$, $c$, or $rc$. There is a fourth extreme point (shown in gray) corresponding to a non-principal upper set.
  The cross section $Q_0(\mL)$ (right). There are three extreme points, which
  correspond to the vectors $d(-,r), d(-,c), d(-,rc)$.}
\label{fig-example}
\end{figure}
\begin{remark}
    Note that approximating uniformly infinities in the matrix $d$
    with a big number $M$ can be done in general.
    This is helpful since the duality theorem is easier to illustrate
    for matrices with finite entries.
Of course, the proof of the duality theorem in section \ref{section 6} goes through with coefficients in $(-\infty,\infty)$.  
(Note also that the Develin-Sturmfels version \cite{sturmfels} of the adjunction between the tropical column span and the tropical row span is exactly about matrices with finite entries, whereas the version of \cite{cgq02} deals with matrices with possibly infinite entries.)

However as we already saw in the example  replacing $\infty$ with $M$ in a directed metric  $d$ that defines a probabilistic language model gives a metric that is no longer a language model. The limit of polyhedra for $M\to \infty$ will give the polyhedra for the original metric.

\end{remark}

\begin{figure}[htbp]
  \begin{center}
    \begin{minipage}[b]{0.25\textwidth}
      \renewcommand{\baryx}{u_1}
\renewcommand{\baryy}{u_2}
\renewcommand{\baryz}{u_3}

\begin{tikzpicture}%
[scale=0.45,>=triangle 45
,vtx/.style={mygreen},
ray/.style={myred}]
\equilateralnobars{7}{100};
\def\myscale{1.8}
\def\bigM{3}

\barycenter{e1}{\expo{0}}{0}{0};
\barycenter{e2}{0}{\expo{0}}{0};
\barycenter{e3}{0}{0}{\expo{0}};

\barycenter{A1}{\expo{0}}{\expo{-\bigM}}{\expo{-2/\myscale}};
\barycenter{A2}{\expo{-\bigM}}{\expo{0}}{\expo{-3/\myscale}};
\barycenter{A3}{\expo{-\bigM}}{\expo{-\bigM}}{\expo{0}};

\barycenter{B1}{\expo{0}}{\expo{\bigM}}{\expo{\bigM}};
\barycenter{B2}{\expo{\bigM}}{\expo{0}}{\expo{\bigM}};
\barycenter{B3}{\expo{2/\myscale}}{\expo{3/\myscale}}{\expo{0}};

\filldraw[gray,draw=black,opacity=0.5,very thick] (A1) -- (B3) -- (A2) -- (B1) -- (A3) -- (B2) -- cycle;

\filldraw[vtx] (A1) circle (0.75ex) node[left] {$r$}; %
\filldraw[vtx] (A2) circle (0.75ex) node[right] {$c$}; %
\filldraw[vtx] (A3) circle (0.75ex) node[above,right] {$rc$}; %

\filldraw[gray] (B1) circle (0.75ex) node[left] {}; 
\filldraw[gray] (B2) circle (0.75ex) node[right] {};
\filldraw[gray] (B3) circle (0.75ex) node[above,right] {}; 
\end{tikzpicture}
\end{minipage}\hskip 5em
  \begin{minipage}[b]{0.25\textwidth}        \renewcommand{\baryx}{z_1}
\renewcommand{\baryy}{z_2}
\renewcommand{\baryz}{z_3}
\begin{tikzpicture}%
[scale=0.45,>=triangle 45
,vtx/.style={mygreen},
ray/.style={myred}]
\equilateralnobars{7}{100};

\def\myscale{1.8}
\def\bigM{3}

\barycenter{e1}{\expo{0}}{0}{0};
\barycenter{e2}{0}{\expo{0}}{0};

\barycenter{e3}{0}{0}{\expo{0}};
\barycenter{A1}{\expo{0}}{\expo{-\bigM}}{\expo{-\bigM}};
\barycenter{A2}{\expo{-\bigM}}{\expo{0}}{\expo{-\bigM}};
\barycenter{A3}{\expo{-2/\myscale}}{\expo{-3/\myscale}}{\expo{0}};

\barycenter{B1}{\expo{\bigM}}{\expo{0}}{\expo{3/\myscale}};
\barycenter{B2}{\expo{0}}{\expo{\bigM}}{\expo{2/\myscale}};
\barycenter{B3}{\expo{0}}{\expo{0}}{\expo{-\bigM}};
\filldraw[vtx] (A1) circle (0.75ex) node[left] {$r$}; %
\filldraw[vtx] (A2) circle (0.75ex) node[right] {$c$}; %
\filldraw[vtx] (A3) circle (0.75ex) node[above] {$rc$};%

\filldraw[gray] (B1) circle (0.75ex) node[left] {}; %
\filldraw[gray] (B2) circle (0.75ex) node[right] {}; %
\filldraw[gray] (B3) circle (0.75ex) node[above] {};%

\filldraw[gray,draw=black,opacity=0.5,very thick] (A2) -- (B3) -- (A1) -- (B1) -- (A3) -- (B2) --cycle ;
\end{tikzpicture}
\end{minipage}
\end{center}
\caption{The duality between the columns and row spaces of metric matrices (\Cref{prop-antiisom} and~\Cref{th-4}) illustrated. On the right $\Im(d^M_{\min})$ and on the left $\Im((d^M_{\min})^t)$}
\label{fig-example-bigM}
\end{figure}
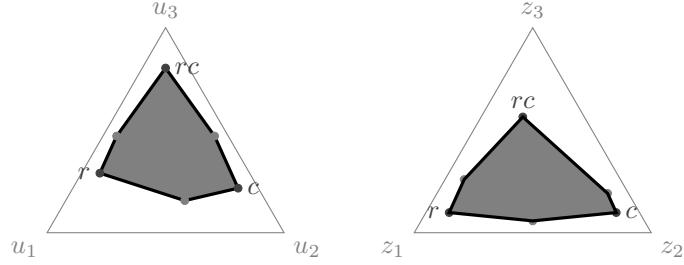

\begin{remark}
\label{re word duality}
We have said that we can 
encode the meaning of $a_k$ by $d(-,a_k)$. If $a_k$ is a word 
this contains very little information.
We already addressed the solution to this problem in section \Cref{Section Semantic spaces}. Another solution is, using the duality between $P(\mL)$ and $\widehat{P}(\mL)$ explained previously and in particular \Cref{eq duality 1}, \Cref{eq duality 2}.
\end{remark}

\section{Extremal Rays in terms of text vectors}

 We have seen in \Cref{th-1} that the original texts in $\mL$, mapped by the Yoneda isometric embedding $Y:\mL \to P(\mL)$, appear as extremal rays (corresponding to principal lower sets) in the polyhedron $P(\mL)$ and the polyhedral cone $Q(\mL)$. 
 As proven in \Cref{prop-7},
 there are in general many other extremal rays of $P(\mL)$ corresponding to connected lower sets of $\mL$. 
Nevertheless extremal rays in the image of $Y$,  $(\min,+)$ generate $P(\mL)$ as we have already seen in \Cref{coro-4} where we showed that
\begin{equation}
x\in P(\mL) \iff
x=\oplus_j D(Y(a_j),x) \odot Y(a_j).
\end{equation}

Recall that we think of $Y(a_k):=d(-,a_k)\in P(\mL)$ as encoding the meaning of text $a_k$ according to the statistical semantics principal.

Recall \Cref{def saturation graph}, where we introduced the saturation graph $S(x)$ for $x\in P(\mL)$. We shall
also consider the {\em undirected} saturation graph, obtained by forgetting the orientation
of the edges in $S(x)$.

\begin{proposition}
\label{prop terminal}
A vector $x \in P(\mL)$ can be written as a tropical linear combination of terminal elements of its saturation graph $S(x)$.
Specifically if $b_1,\dots b_k$ are the terminal elements in $S(x)$, then 
\begin{equation}
  x=\oplus_j D(Y(b_j),x) \odot Y( b_j).
\label{e-terminal}
\end{equation}
Moreover, if the saturation graph of $x$ has $s$ undirected connected components, this vector belongs to a face of $P(\mL)$ of dimension $s$.

\end{proposition}
\begin{proof}
  Let $x\in P(\mL)$. If $i$ is not terminal in the saturation graph $S(x)$, then, there is a $j_1\neq i$ such that $x_i=d_{i,j_1}+x_{j_1}$.
Similarly, if $j_1$ is not terminal, there is  $j_2\neq j_1$ such that 
$x_{j_1}=d_{j_1,j_2}+x_{j_2}$  and thus  $x_i=d_{i,j_1}+d_{j_1,j_2}+x_{j_2}$. Continuing this way we get 
 $x_i=d_{i,j_1}+d_{j_1,j_2}+\dots d_{j_{n-1},j_n}+x_{j_n}$ where $a_{j_n}$ is a terminal element of the saturation graph $S(x)$. 
 
We know that this stops at a terminal element because there are no cycles of positive weight
in the digraph of $d$.
(This is the graph that has an edge from $i$ to $j$ with weight $d_{i,j}$ when $d_{i,j}$ is not infinity.)

Using the triangular inequality, we deduce that $x_i\geq \oplus_{j \in T} d_{i,j} \odot x_{j}$,
where the sum is taken over the set $T$ of terminal nodes of $S(x)$,
and so, $x\geq \oplus_{j \in T} D(Y(a_i),Y(a_{j})) + D(Y(a_{j}),x)$.
Conversely, by definition of $P(\mL)$, $x\leq \oplus_k d_{ik} \odot x_k =
\oplus_{k} D(Y(a_i),Y(a_k)) \odot D(Y(a_k),x)$ where now the sum is taken
over all the indices $k$ (possibly non terminal). This entails that~\eqref{e-terminal} holds.

Finally, arguing as in the proof of~\Cref{prop-char-extreme},
we get the rank of the family of active constraints at point $x$ is given by
the number $s$ of connected component of the undirected saturation graph of $x$.
Hence, $x$ belongs to a face of dimension $s$.
\end{proof}
\begin{remark}
Note that \Cref{prop terminal} holds for a general directed metric space $\mL$ and not just for $(\mL,d)$ a probabilistic language model.
\end{remark}
We can now find explicit $(min,+)$ expressions for generators of extremal rays corresponding to non principal lower sets.
\begin{proposition}
Let $\mL$ be a probabilistic language model with the empty text $a_0$ included. Let $x$ denote an extremal ray corresponding to the lower set generated by $\{b_1,\dots b_n\}$. Then 
\begin{equation}
x=\oplus_i \logPr(b_i) \odot  Y(b_i).
\end{equation}
\end{proposition}
\begin{proof}

From \Cref{prop terminal} we have
$$x=\oplus_jx(b_j)\odot Y(b_j).$$
From \cref{co coord extremal} we have 
$x(b_j)=-\log\frac{1}{\Pr(b_j)}=\logPr(b_j)$.
This proves the result.
\end{proof}
\begin{remark}
We point out that  the terminal elements 
$b_1,\dots b_k$ of $S(x)$
function like an orthonormal basis with respect to $D$, namely 
\begin{equation}
D(Y(b_i),Y(b_j))=d(b_i,b_j)=\infty \text{ if } i\neq j.
\end{equation}

So for example if we know that 
there are $\lambda_j$ such that 

$x=\oplus_j \lambda_j \odot Y(b_j)$ 
Then 
$D(Y(b_i),x)=\oplus_j \lambda_j \odot D(Y(b_i),Y(b_j))=\lambda _i$.
\end{remark}

\begin{proposition}
Let $T\in [0,\infty)$ be a parameter which will be called temperature. Consider $x\in P(\mL)$ an extremal ray and let
$x=\oplus_j D(Y(b_j),x) \odot Y( b_j)$ where $b_j$ are the terminal elements of the saturation graph S(x). 
Let $v_j\coloneqq e^{Y( b_j)}$.
Then we have 
\begin{equation}
x =\lim_{T \to 0} -T \log (\sum_j e^{-D(Y(b_j),x)/T} e^{Y( b_j)})
\end{equation}

and therefore, for small $T$
\begin{equation}
e^{-x/T}\approx \sum_j e^{-D(Y(b_j),x)/T}v_j
\end{equation}

\end{proposition}
\begin{proof}
Recall the identity 
\begin{equation}
\lim_{T \to 0} -T \log(e^{-y/T} +e^{-z/T})=\min\{y,z\}.
\end{equation}
 If $x \in P(\mL)$, by the previous proposition $x=\oplus_j D(Y(b_j),x) \odot Y( b_j)$ where $b_j$ are terminal elements. Then 
we have 
\begin{equation}
x =\lim_{T \to 0} -T \log \sum_j e^{-D(Y(b_j),x)/T} e^{Y( b_j)}
\end{equation}
and if we put $v_j\coloneqq e^{Y( b_j)}$,
then for small $T$,  we get

\begin{equation}
e^{-x/T}\approx \sum_j e^{-D(Y(b_j),x)/T}v_j \enspace .\qedhere
\end{equation}
\end{proof}

\section{$P^-(\mL)$ as the lattice completion of the  Isbell completion}

We have seen that $P(\mL)$ and $\widehat{P}(\mL)$ generalize the lower set and upper set completions respectively from the poset $\mL$ to the directed metric space $(\mL,d)$, at least in the case where 
$\mL$ contains the empty text $a_0$ which is the bottom element. 

However there is another completion of a poset, called the Dedekind-MacNeille completion (which also generalizes the so called notion of formal concepts). 

It is known that the generalization of the Dedekind MacNeille completion from posets to directed metric spaces is the so called Isbell completion, which is the fixed part of the Isbell adjunction. 

This is also relevant to our situation as it turns out to be defined by $d_{\max}$.

This was studied in \cite{Lawvere73, willerton2013tight} with $[0,\infty]$ coefficients. In that case the Isbell completion is identified with the directed tight span of Hirai and Koichi \cite{Hirai2012}. 

We will instead define the Isbell adjunction using the extended semi ring $[-\infty,\infty]$ as we did with the $d_{\min}$ adjunction in section~\ref{section 6}.

Recall that in section~\ref{section 6} \Cref{rem absorbing} we explained the conventions for working with 
$(\min,+)$ and $(\max,+)$ on $[-\infty,\infty]$. We use the same here.

Given $x: \mL \to [-\infty,\infty]$   and  
$y: \mL \to [-\infty,\infty]$ 
define $d_{\max}$ and $d^t_{\max}$ by
\begin{equation}
d_{\max}(x)_i:=\max_j\{d_{i,j}+x_j\} \textit{ and }
d^t_{\max}(y)_j:=\max_i\{d_{i,j}+y_i\}
\end{equation}

Extending the definition in  \cite{Lawvere73, willerton2013tight} by using $[-\infty,\infty]$ coefficients we have that

\begin{definition}

The Isbell adjunction is the pair of maps $(L,R)$ defined as follows. If $x: \mL \to [-\infty,\infty]$   and  
$y: \mL \to [-\infty,\infty]$  then
\begin{equation}
L(x):=d_{\max}(-x) \text{ and }
R(x):=d^t_{\max}(-y)
\end{equation}

Or in coordinates
\begin{equation}
L(x)_i:=\max_j\{d_{i,j}-x_j\} \text{ and }
R(x)_j:=\max_i\{d_{i,j}-y_i\}
\end{equation}
\end{definition}

Recall from section~\ref{section 6} that the Funk metric  $D$, is still well defined by $D(x,y):=\max_i\{y_i-x_i|x_i\neq \infty \}$.
We also denoted by $D^t$ the transpose metric with $D^t(x,y):=D(y,x)$.

  \begin{remark}Note that 
  \begin{equation}
  L(x)_i:=\max_j\{d_{i,j}-x_j\}=
  D(x,d(a_i,-))=D(x,\widehat{Y}(a_i))
  \end{equation}
  and
   \begin{equation}
R(y)_j:=\max_i\{d_{i,j}-y_i\}=D(y,d(-,a_j))=D(y,Y(a_j))
 \end{equation}

\end{remark}

The pair $(L,R)$ forms an adjunction in the categorical or metric sense:

\begin{proposition}\label{prop-15}
If  $x: \mL \to [-\infty,\infty]$ and  $y: \mL \to [-\infty,\infty]$  then 
we have $D^t(Lx,y)=D(x,Ry)$.
\end{proposition}

\begin{proof}

$D^t(Lx,y)=D(y,Lx)=\max_i\{\max_j\{d_{i,j}-x_j\}-y_i\}=\max_j\{\max_i\{d_{i,j}-y_i\}-x_j\}=D(x,Ry)$.
\end{proof}

We now have the following
\begin{proposition}
We have $LRL=L$ and $RLR=R$ which implies that $LR$ and $RL$ are idempotent. 
\end{proposition}

\begin{proof}

This follows from the fact that $D^t(Lx,y)=D(x,Ry)$.
Indeed 
$D^t(LRLx,Lx)=D(RLx,RLx)=0$ and 
$D^t(Lx,LRLx)=D(RLx,RLx)=0$. 
Therefore $LRLx=Lx$. The equality $RLR=R$ is shown analogously.
\end{proof}

Let us now compute the fixed parts of the adjunction $\Fix(LR)$ and $\Fix(RL)$.

\begin{proposition}
We have $\Fix(LR)=\Im(L)=\Im(d_{\max})$ and $\Fix(RL)=\Im(R)=\Im(d_{\max}^t)$.
\end{proposition}
\begin{proof}
 This follows from the fact that $LRL=L$. Indeed clearly $\Im(L)\subset \Fix(LR)$. 
Moreover $\Fix(LR)\subset \Im(L)$ since $LR(y)=y$ says that $y\in \Im(L)$.
\end{proof}

As before we have the following 
 \begin{proposition}
We have that $$L:\Fix(RL)=\Im(R)=\Im(d_{\max}^t) \to \Fix(LR)=\Im(L)=\Im(d_{\max})$$ and 
$$R:\Fix(LR)=\Im(L)=\Im(d_{\max}) \to \Fix(RL)=\Im(R)=\Im(d^t_{\max})$$
are anti-isomorphisms. In other words they are one to one and onto and inverses.  They are isometries, namely
$D(Lx,Lx')=D^t(x,x')$.
Finally we have
\begin{equation}
L(\lambda \odot x)=-\lambda \odot L(x)
\textit{ and } 
L(x \oplus_{\min} y)=L(x)\oplus_{\max} L(y)
\end{equation}
and similarly for $R$.
\end{proposition} 

\begin{proof}
First let us check that $L$ and $R$ are one to one and onto. 
Consider $x,x'\in \Fix(RL)$. If $L(x)=L(x')$ then $RL(x)=RL(x')$ and therefore $x=x'$.
Also if $y\in \Fix(LR)$ then $y=L(R(y))$.

Moreover $D(Lx,Lx')=D_{op}(RLx,x')=D_{op}(x,x')=D(x',x)$

Next we check the tropical antilinearity.

$L(\lambda \odot x)_i=\max_j\{d_{i,j}-\lambda-x_j\}=\max_j\{d_{i,j}-x_j\}-\lambda=L(x)_i-\lambda=(-\lambda \odot L(x))_i$.

Moreover 

$L(x \oplus_{\min} y)_i=\max_j\{d_{i,j}-\min\{x_j,y_j\}\}=
\max_j\{d_{i,j}+\max\{-x_j,-y_j\}=
\max_j\{\max\{d_{i,j}-x_j,d_{i,j}-y_j\}\}=
\max \{\max_j\{d_{i,j}-x_j\}, \max_j\{d_{i,j}-y_j\}\}
=(L(x)\oplus_{\max} L(y))_i$.

\end{proof}

We have that the tropical linear space $\Im(d_{\max})$ is anti isomorphic to  $\Im(d^t_{\max})$ by the two inverse maps
$$R: \Im(d_{\max}) \to \Im(d^t_{\max}) \textit{ and }
L: \Im(d^t_{\max}) \to \Im(d_{\max}).$$ 

\begin{proposition}

 The Yoneda isometric embedding $Y:\mL \to P(\mL)$ given by $Y(a):=d(-,a)$ and the co-Yoneda isometric embedding $\widehat{Y}: \mL \to \widehat{P}(\mL)$ given by $\widehat{Y}(a):=d(a,-)$, are compatible with the anti-isomorphisms $L$ and $R$ above, in the sense that 
\begin{equation}
\widehat{Y}(a)=R(Y(a)) \textit{ and } Y(a)=L(\widehat{Y}
(a)).
\end{equation}
\end{proposition}

\begin{proof}
We have 
$d_{i,j}\leq d_{i,k}+d_{k,j}$, therefore
$$L(\widehat{Y}(a_k))_i=L(d(a_k,-))_i=\max_j\{d_{ij}-d_{k,j}\}=d_{i,k}=d((-,a_k)_i=Y(a_k)_i.$$
Analogously
\[R(Y(a_k))_j=R(d(-,a_k))_j=\max_i\{d_{ij}-d_{i,k}\}=d_{k,j}=d((a_k,-)_j=\widehat{Y}(a_k)_j\enspace .\qedhere\]
\end{proof}
As mentioned earlier according to a theorem of Willerton \cite{willerton2013tight} 
\begin{theorem}
The directed tight span of Hirai and Koichi \cite{Hirai2012} is the same as the fixed parts of the Isbell adjunction when  using $[0,\infty]$ coefficients and the truncated $\max$ operations .

\end{theorem}

We denote the Isbell completion with $[-\infty,\infty]$ coefficients by $\tilde{I}(\mL)$.

Let us finally explore  the relation between $P(\mL)=\Im(d_{\min})$ and $\tilde{I}(\mL)=\Im(d_{\max})$.

\begin{proposition}\label{prop-lattice-completion}
    The polyhedron $P(\mL)=\Im(d_{\min})$ is the lattice completion of $\Im(d_{\max})$ when using $(-\infty,\infty]$ coefficients and $P^-(\mL)$ is the lattice completion of $\Im(d_{\max})$ when using $[-\infty,\infty]$ coefficients.
\end{proposition}

\begin{proof}
Recall that since $d_{\min}^2=d_{\min}$ we have 
$\Im(d_{\min})=\{x|dx=x\}$. 
Moreover if $\Id$ is the $(\min,+)$ identity matrix, namely $\Id_{i,i}=0$ and $\Id_{i,j}=\infty$ for $i\neq j$, then, 
since $d_{i,}=0$, we have  $d\leq \Id$ and therefore  we always have $dx\leq x$.

It follows that $x=dx \iff x\leq dx$. 

We want to show that if $x\leq dx$ and $y \leq dy$ then 
\[\max\{x,y\} \leq d(\max\{x,y\})\] which will imply that $\max\{x,y\} \in \Im(d_{\min})$.

Indeed if $\max\{x,y\}=x$ then $d(\max\{x,y\})=dx\geq x\geq y$
and if $\max\{x,y\}=y$ then $d(\max\{x,y\})=dy\geq y\geq x$, therefore 
$x\leq d(\max\{x,y\})$ and $y \leq d(\max\{x,y\})$ which implies that $\max\{x,y\}\leq d(\max\{x,y\}) $.

We have shown therefore that $\Im(d_{\min})$ is closed under the $\max$ operation. 
Both $\Im(d_{\min})$ and $\Im(d_{\max})$ are generated by the vectors $d(-,a_i)$ therefore the result is proved.
\end{proof}
\begin{example}
To illustrate the difference between $\Im(d_{\min})$ and $\Im(d_{\max})$ (albeit for a  symmetric and finite metric) we provide the following example: Consider the discrete metric on three points
\begin{align}
d_2 = \left(\begin{array}{ccc}
 0 & 1& 1\\
 1 & 0 & 1 \\\
 1 & 1 & 0
\end{array}\right)
\label{e-def-d2}
\end{align}
The associated (min,+)- module $\mathcal{P}(d_2)$ is shown
on~\Cref{fig-example2}, left and the $(\max,+)$ module on the right. 
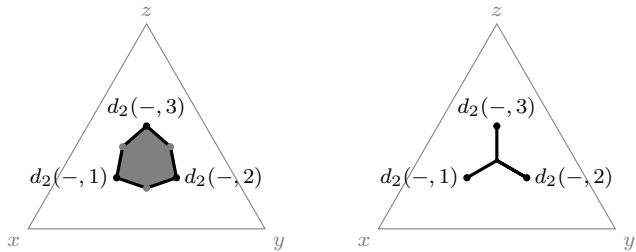
\begin{figure}[htbp]
\begin{center}\footnotesize
\begin{minipage}[b]{0.25\textwidth}
\begin{tikzpicture}%
[scale=0.45,>=triangle 45
,vtx/.style={mygreen},
ray/.style={myred}]
\equilateralnobars{7}{100};
\barycenter{e1}{\expo{0}}{0}{0};
\barycenter{e2}{0}{\expo{0}}{0};
\barycenter{e3}{0}{0}{\expo{0}};
\barycenter{A1}{\expo{0}}{\expo{-1}}{\expo{-1}};
\barycenter{A2}{\expo{-1}}{\expo{0}}{\expo{-1}};
\barycenter{A3}{\expo{-1}}{\expo{-1}}{\expo{0}};
\barycenter{B1}{\expo{0}}{\expo{1}}{\expo{1}};
\barycenter{B2}{\expo{1}}{\expo{0}}{\expo{1}};
\barycenter{B3}{\expo{1}}{\expo{1}}{\expo{0}};
\barycenter{PSEUDO}{\expo{2}}{\expo{3}}{\expo{0}};

\filldraw[gray,draw=black,opacity=0.5,very thick] (A1) -- (B3) -- (A2) -- (B1) -- (A3) -- (B2) -- cycle;

\filldraw[gray] (B1) circle (0.75ex);
\filldraw[gray] (B2) circle (0.75ex);
\filldraw[gray] (B3) circle (0.75ex);

\filldraw (A1) circle (0.75ex) node[left] {$d_2(-,1)$};
\filldraw (A2) circle (0.75ex) node[right] {$d_2(-, 2)$};
\filldraw (A3) circle (0.75ex) node[above] {$d_2(-, 3)$};

\end{tikzpicture}
\end{minipage}\hskip 5em
\begin{minipage}[b]{0.25\textwidth}
\begin{tikzpicture}%
[scale=0.45,>=triangle 45
,vtx/.style={mygreen},
ray/.style={myred}]
\equilateralnobars{7}{100};
\barycenter{zero}{\expo{0}}{\expo{0}}{\expo{0}};
\barycenter{e1}{\expo{0}}{0}{0};
\barycenter{e2}{0}{\expo{0}}{0};
\barycenter{e3}{0}{0}{\expo{0}};
\barycenter{A1}{\expo{0}}{\expo{-1}}{\expo{-1}};
\barycenter{A2}{\expo{-1}}{\expo{0}}{\expo{-1}};
\barycenter{A3}{\expo{-1}}{\expo{-1}}{\expo{0}};
\barycenter{B1}{\expo{0}}{\expo{1}}{\expo{1}};
\barycenter{B2}{\expo{1}}{\expo{0}}{\expo{1}};
\barycenter{B3}{\expo{1}}{\expo{1}}{\expo{0}};
\barycenter{PSEUDO}{\expo{2}}{\expo{3}}{\expo{0}};

\draw[very thick] (A3) -- (zero) -- (A2) -- (zero) -- (A1) ;

\filldraw (A1) circle (0.75ex) node[left] {$d_2(-, 1)$};
\filldraw (A2) circle (0.75ex) node[right] {$d_2(- , 2)$};
\filldraw (A3) circle (0.75ex) node[above] {$d_2(-, 3)$};

\end{tikzpicture}
\end{minipage}
\end{center}
\caption{Tropical module generated by the discrete metric $d_2$ of~\Cref{e-def-d2}. The pseudo-vertices (vertices of the polyhedral complex that do not arise from tropical generators) are shown in gray. (left) The (max,+)-span (right).}
\label{fig-example2}
\end{figure}

\end{example}
\section{Some comments about Probabilistic Language Models}
We would finally like to gather some comments about how to interpret probabilistic language models
$(\mL,\leq, \Pr)$ and what they imply. Some of these were stated already in section \Cref{Section Semantic spaces}

\begin{enumerate}

\item
We note that the construction of $P(\mL)$ explains why it is natural to have vectors in a problem of language.
In fact we naturally get Boltzmann weighted linear combinations \Cref{eq text vector}, \Cref{eq system Yoneda}, \Cref{eq system Yoneda} which is what is 
introduced by hand in the attention layers of the transformer and the final layer where the distribution over possible next words is determined 
\item

If  the transformer is learning 
$\widehat{P}(\mL)$ or equivalently $\widehat{Q}(\mL)$ it would be learning a convex body which could explain why its training is efficient in the first place.

\item

Assuming that the transformer is learning the 
polyhedron $\widehat{P}(\mL)$ or equivalently $\widehat{Q}(\mL)$ it would be learning an effective representation of Yoneda  embeddings of texts. 
This can then be  interpreted as  solving the huge $(\min,+)$ linear systems in 
\Cref{eq system Yoneda}, \Cref{eq system coYoneda}, (\Cref{prop linear system}).

\item 
The duality explained in section~\ref{section 6} between $P(\mL)$ and $\widehat{P}(\mL)$ shows how to resolve the paradox that both $d(-,a_k)$ and 
$d(a_k,-)$ should equally well encode the meaning of a text $a_k$, given a probabilistic language model $(\mL,\leq,\Pr)$. This is most striking when $a_k$ is a single word. 
In that case $d(-,a_k)$ is supported only on $a_k$, but we have that $-d(-,a_k)=-d(a_k,-)$.
This was explained in \Cref{section 6}, \Cref{prop duality formula}, \Cref{re word duality}. 
It also showcases the notion that the meaning of a text $a_k$ is not just encoded by $d(-,a_k)$ or $d(a_k,-)$ but by the whole ambient spaces $P(\mL)$ and $\widehat{P}(\mL)$ respectively.

\end{enumerate}
\appendix
\section{Categorical interpretation}

The metric polyhedra $(P(\mL),D)$ and $(\widehat{P}(\mL),D^t)$, as well as the polyhedral cones $Q(\mL)$ and $\widehat{Q}(\mL)$,  arise from a categorical point of view \cite{Lawvere73, willerton2013tight, BTV2021}.
In fact all constructions have a categorical interpretations and we will briefly explain these here.

To begin with,  we  can consider the probabilistic language model $(\mL,d)$ to be a category enriched over the monoidal closed category $(-\infty,\infty]$, with monoidal structure given by addition and Hom given by considering
$(-\infty,,\infty]$ as a poset with the opposite of  the usual order of numbers. The  Hom between objects $a_i$ and $a_j$ in $\mL$, is $d(a_i,a_j)$ and the triangle inequality is the composition of morphisms.

This construction (using $[0,\infty]$ instead of $(-\infty,\infty]$) was explained in \cite{BTV2021})

Then 
$P(\mL)$ is the category of presheaves on $\mL$ 
namely the category of enriched functors 
$\mL^{\op} \to (-\infty,\infty]$ where $(-\infty,\infty])$ is considered as a category enriched over itself with internal Hom given by the directed  metric $d_\R$ on $(-\infty,\infty]$ where  $d_\R(s,t)=t-s$.
This follows from the fact that
we can think of the points $x\in P(\mL)$ as non-expansive functions on $\mL$ as we have seen in \Cref{prop-4}. Indeed
\begin{equation}
P(\mL)=\{x:(\mL,d^t) \to ((-\infty,\infty],d_\R)  | x \text{ is non-expansive.}\} 
\end{equation}

Moreover, the Funk directed metric $D$ on $P(\mL)$ is the Hom on presheaves.

The isometric embedding $Y:\mL \hookrightarrow P(\mL)$ is 
the Yoneda embedding, $Y(a_k)$ is a representable presheaf and the fact that $x_i=x(a_i)=D(Y(a_i),x)$, is the Yoneda lemma. 

On the other hand $\widehat{P}(\mL)$ is the category of co-presheaves and $\widehat{Y}$ is the co-Yoneda embedding. 
The tropical anti-isomorphims between $P(\mL)$ and $\widehat{P}(\mL)$ as already explained follows from an adjunction between $A(x):=d_{\min}(-x)$ and $B(y)=d^t_{\min}(-y)$.

Finally it was proven in \cite{willerton2013tight} that the directed tight span  $DTS(\mL)$ (defined in \cite{Hirai2012}) is the Isbell completion, with $[0,\infty]$ as enriching category,  of the enriched category $\mL$. Namely the fixed part of the Isbell adjunction which is given by 
$(L(x))_i:=\max_j\{d_{i,j}-x_j\}$ and $(R(y))_j:=\max_i\{d_{i,j}-y_i\}$ (where we use trancated difference so the result is always positive). 
We instead define the Isbell adjunction with enriching category $[-\infty,\infty]$. 

The fact that the category of presheaves $P(\mL)$ is the $(\min,+)$ span of the images of the Yoneda embedding reflects the fact that  colimits are given by $\min$ and every presheaf is a weighted colimit of representables.

On the other hand the Isbell completion is given by presheaves which are weighted limits of representables since limits are given by $\max$ and it is smaller that $P(\mL)$ since in general not every presheaf is such a  weighted limit. 

\section{Syntax to Semantics and Morita equivalence}

The problem of encoding allowed (with some probability)  sequences of symbols, by some mathematical structure can be located in the realm of a very basic duality in mathematics.

Traditionally language has been modeled as a monoid generated by words. 
We can go from the monoid to a poset by 
considering the monoid as a category with one object and arrows corresponding to texts and constructing the factorization category (also called the twisted arrow category). This produces exactly the poset of texts with the subtext order as we have used in our probabilistic language model.

Considering the subtext poset  makes it easier to add probabilities and we are led naturally to the probabilistic language model we defined which is a special case of a directed metric space. In Appendix A we saw that this is an enriched category.

In the monoid case we consider that the meaning of a text eg ``red'' is given by the ideal generated by red which contains all texts containg red. 

In the poset case it is the same, where ideals and filters correspond to principal lower and upper sets.

This is a mathematical incarnation of the distributional semantics principle. 

Now there is a very general and basic concept of duality in mathematics that in the commutative case takes the form of a duality between algebra and geometry. 

The most basic case is, given a commutative algebra, to consider the space of (prime) ideals.

This is called the spec and can be thought of as a space on which the algebra of functions is the commutative algebra we started with. For example 
if we consider the algebra $\mathbb{C}[x,y]$ of complex polynomials in two  variables then prime ideas are ideals generated by monomials $(x-a)(y-b)$ for any $a,b \in \mathbb{C}$ and therefore the space of ideals is $\mathbb{C}^2$ i.e. the complex plane.
The duality then is between the commutative algebra $\mathbb{C}[x,y]$ and the space of ideals $\mathbb{C}^2$. (This so called  spec construction  is the cornerstone of algebraic geometry.)

We can try to extend this kind of duality for monoids, posets, and for our enriched category.
Ideals in a monoid are modules over the monoid and in general we have to consider modules. Now moving to the case of an algebra, a module over the algebra (a representation) is a presheaf over the corresponding category. This is the category with one object and arrows given by the elements of the algebra. 
In general in a category the presheaves play the role of modules.

In our case modules i.e. presheaves  are the non-expansive maps (\Cref{prop-3}) and the space $P(\mL)$ is the category of modules (the Hom is given by the metric $D$ as already mentioned in Appendix A).

The original  category defines the syntax and the presheaf category can be considered to reflect semantics (see also \cite{BTV2021}). 

In fact just like $\mathbb{C}[x,y]$ gives coordinates on the space of ideals $\mathbb{C}^2$, we could think that 
the language category (the syntax category) provides  coordinates on the category of presheaves (modules) which can be though as the semantic category (in this particular case, for example because the Hom which is the metric $D$ measures semantic similarity).

Now since we can translate between languages, namely the semantics of languages are in some sense the same (approximately) we 
expect that the categories of presheaves on different language categories, should be equivalent. 
This is a well known notion called \textit{Morita equivalence}. We would then expect that enriched categories corresponding to different languages should be Morita equivalent. 
Moreover in that case there are associated invariants (Hochschild homology) which should be semantic invariants.

Investigating and developing this, is a future direction of research.

\end{document}